%% file: main.tex
\documentclass{article}

\usepackage{booktabs}

 \usepackage[dblblindworkshop, final]{neurips_2026}
\workshoptitle{Geometric Distributional Deep Learning (GDDL)}

\usepackage[utf8]{inputenc} 
\usepackage[T1]{fontenc}    
\usepackage{hyperref}       
\usepackage{url}            
\usepackage{booktabs}       
\usepackage{amsfonts}       
\usepackage{nicefrac}       
\usepackage{microtype}      
\usepackage{xcolor}         
\usepackage{amsmath}
\usepackage{amssymb}
\usepackage{amsthm}
\usepackage{graphicx}
\usepackage{xcolor}
\usepackage{tikz}
\usepackage{wrapfig}
\usepackage{algorithm}
\usepackage{algpseudocode}

\usetikzlibrary{arrows.meta,positioning}

\newtheorem{theorem}{Theorem}
\newtheorem{proposition}[theorem]{Proposition}
\newtheorem{lemma}[theorem]{Lemma}

\input{math_commands.tex}

\usepackage{mathtools}
\usepackage{bm}

\usepackage{booktabs}
\usepackage{graphicx}
\usepackage{xcolor}

\usepackage{algorithm}
\usepackage{algpseudocode}

\usepackage{microtype}
\usepackage{xspace}
                                                      
\usepackage{hyperref}                                    \usepackage{cleveref}

\theoremstyle{plain}
\newtheorem{corollary}[theorem]{Corollary}

\theoremstyle{definition}

\theoremstyle{remark}

\usepackage{hyperref}
\usepackage{url}

\title{Discrete Wasserstein Flows for One-Step Generative Modeling}

\author{%
  Alessandro Micheli\\
  Imperial College London \\
  London, United Kingdom \\
 \texttt{a.micheli19@imperial.ac.uk} \\
  \And
  Andrea Zerio \\
  Department of Computer Science, Aalborg University,
  Copenhagen, Denmark;\\
  Centre for Frontier AI Research (CFAR),\\
  Institute of Advanced Intelligence and Computing (IAIC),\\
  A*STAR, Singapore
  \texttt{anze@cs.aau.dk} \\
  \And
  Samir Bhatt \\
  Imperial College London; University of Copenhagen \\
  London, United Kingdom; Copenhagen, Denmark \\
  \texttt{s.bhatt@imperial.ac.uk} \\
}

\begin{document}

\maketitle

\input{sections/abstract}
\input{sections/intro}
\input{sections/background}
\input{sections/method}
\input{sections/experiments}

\input{sections/conclusion}

\newpage

\bibliographystyle{plainnat}
\bibliography{references}

\newpage

\appendix
\input{sections/appendix}

\end{document}

%% file: math_commands.tex
\usepackage{amsmath,amsfonts,bm}

\newcommand{\X}{\mathcal{X}}

\def\eqref#1{equation~\ref{#1}}

\def\1{\bm{1}}

\DeclareMathAlphabet{\mathsfit}{\encodingdefault}{\sfdefault}{m}{sl}
\SetMathAlphabet{\mathsfit}{bold}{\encodingdefault}{\sfdefault}{bx}{n}



%% file: sections/abstract.tex
\begin{abstract}
We introduce a new framework for one-step generative modelling on finite state spaces. To extend drifting beyond continuous domains, we use  discrete Wasserstein geometry to define a target-relative KL gradient flow over the transitions of a reversible Markov kernel. We realize this probability flow at the particle level through Markov jumps and amortize the resulting transport updates into a latent-conditioned generator, so that the iterative dynamics are required only during training while inference remains one-step. In a controlled setting where the underlying distributions and transport dynamics can be computed exactly, we verify KL dissipation, consistency between the particle dynamics and the probability flow, and the predicted numerical scaling. We further show that a finite-capacity neural generator can track these exact transport targets while retaining one-step generation. These results validate the basic construction and provide a foundation for scaling Discrete Drifting to structured discrete data.
\end{abstract}

%% file: sections/intro.tex
\section{Introduction}
\label{sec:introduction}

Diffusion and flow-based generators construct samples through iterative inference-time transport, whereas Drifting Models \citep{deng2026drifting} amortize this evolution into training, enabling one-step generation. Extending this idea to discrete data is not straightforward. In Euclidean space, drifting relies on infinitesimal particle displacements 
\(x \mapsto x + \eta V(x)\); on a finite state space \(\mathcal{X} = \{x_1,\ldots,x_N\}\), however, 
particles cannot move infinitesimally and probability must instead be transferred between distinct states. This obstruction is also geometric: when the underlying state space is finite, the ordinary Wasserstein-$2$ metric does not induce a useful infinitesimal transport geometry on the corresponding space of probability distributions. Indeed, transporting $O(\delta t)$ mass across a positive distance produces \(W_2 = O(\sqrt{\delta t})\),
yielding a divergent metric derivative as $\delta t \to 0$ \citep{Maas2011}. A discrete analogue of drifting therefore requires both an appropriate transport geometry and a particle dynamics compatible with it.

We build both from the discrete Wasserstein geometry of \citet{Maas2011}. An irreducible, reversible Markov kernel $K$ defines the admissible local transitions on $\mathcal{X}$, while the resulting geometry provides a notion of gradient flow for probability distributions on the finite state space. Choosing \(\mathcal{F}_p(q) = \mathrm{KL}(q \Vert p), \)
with $p$ the data distribution, yields a flow that redistributes probability along the transitions of $K$ so as to decrease the target-relative KL. We interpret this redistribution at the particle level through a Markov jump process, providing the discrete counterpart of the continuous drift $V(x)$. Crucially, these jumps are used only during training to construct updated targets: unlike discrete diffusion models \citep{campbell2022,lou2024discrete}, generation does not require simulating a Markov process at inference time. Instead, a latent-conditioned generator is trained to absorb the successive transport updates and directly produce samples in a single forward pass.

Our contributions are threefold: (1) we formulate drifting on finite state spaces using discrete Wasserstein geometry; (2) we derive a Markov-jump particle realization of the resulting probability flow that can be used to construct stop-gradient training targets; and (3) we provide a controlled empirical validation that separately tests the underlying discrete transport and its neural amortization.

%% file: sections/background.tex
\section{Background}
\label{sec:background}

We consider probability evolutions on a finite space
\(\X = \{x_1,\ldots,x_N\},\)
equipped with a Markov transition kernel
\(K\in\mathbb{R}^{N\times N}\), that is, a non-negative, row-stochastic real matrix. We assume throughout that \(K\) is irreducible and reversible. In
particular, irreducibility guarantees the existence of a unique invariant
probability measure \(\pi\in\mathbb{R}^{N}\), with strictly positive
entries, satisfying $\pi^\top K = \pi^\top.$
We denote by $\mathcal{P}(\X)$ the set of probability densities on \(\X\) relative to \(\pi\). We further
write $\mathcal{P}_{*}(\X)$ for the subset of strictly positive densities. Throughout the paper, the symbol \(\rho\) is used for densities relative to
\(\pi\), whereas \(p\) denotes ordinary probability masses. The two
representations are related by
\begin{equation}
\label{eq:density-probability-relation}
    p(x)=\pi(x)\rho(x),
    \qquad x\in\X,
\end{equation}
or, equivalently, $p = \rho\odot\pi$, where \(\odot\) denotes componentwise multiplication.

A canonical example of the probability evolutions considered in this work
is the discrete heat equation
\begin{equation}
\label{eq:discrete-heat-equation}
    \frac{\mathrm{d}}{\mathrm{d}t}\rho_t
    =
    (K-I)\rho_t.
\end{equation}
This dynamics is widely used as a forward noising process in discrete
generative modelling~\citep{campbell2022,lou2024discrete}. Our objective is to interpret evolutions such as
\eqref{eq:discrete-heat-equation} as gradient flows of functionals defined
on \(\mathcal{P}(\X)\), in analogy with variational constructions on
continuous spaces~\citep{Jordan1998,bunne22proximal,
terpin2024learning}. In the
Euclidean setting, probability evolutions can be characterized as gradient
flows of an energy functional \(\mathcal{F}\) with respect to the
\(2\)-Wasserstein geometry. The Jordan--Kinderlehrer--Otto (JKO) scheme
describes this evolution through the proximal recursion
\begin{equation}
\label{eq:jko-scheme-background}
    p_{k+1}
    \in
    \operatorname*{arg\,min}_{p\in\mathcal{P}(\mathbb{R}^{d})}
    \left\{
        \mathcal{F}(p)
        +
        \frac{1}{2\tau}
        W_2^2(p,p_k)
    \right\}.
\end{equation}
A common choice is to consider \(\mathcal{F}\) of the
free-energy form
\begin{equation}
\label{eq:free-energy-form}
    \mathcal{F}
    =
    \mathcal{V}
    +
    \beta\mathcal{H},
\end{equation}
where \(\mathcal{V}\) is a potential energy, \(\mathcal{H}\) is an entropy
functional, and \(\beta\geq 0\) controls the strength of the entropic
contribution.

The standard \(W_2\) distance does not provide a suitable gradient-flow geometry on finite state spaces. This limitation appears even in the simplest non-trivial case of a two-point state space; see~\cite[Lemma 2.1]{RML2026} and~\cite[Remark 2.1]{Maas2011}.
To address this limitation, the transport geometry of~\cite{Maas2011} provides the basis for a discrete analogue of the JKO framework by replacing \(W_2\) with a metric adapted to the transition structure of the Markov kernel \(K\). The construction of this metric parallels the dynamic formulation of
optimal transport due to~\cite{Benamou2000}. Its central ingredient is a \emph{mobility} associated with each edge of the graph induced by \(K\), which determines how the cost of transporting mass between \(x,y\in\X\) depends on their current densities. In this work we use
the logarithmic mean
\[
m(a,b)
:=
\int_0^1 a^s b^{1-s},\mathrm ds,
\]
and define the corresponding mobility by
\[
m_\rho(x,y)
:=
m\bigl(\rho(x),\rho(y)\bigr).
\]
This choice is particularly natural here because, with the corresponding
transport metric, the discrete heat equation in 
\eqref{eq:discrete-heat-equation} is recovered as the gradient flow of
entropy.
The logarithmic mean is one of several possible symmetric mobilities; alternatives include the geometric mean and suitable power-law mobilities~\citep{Maas2011,maaserbar2014,Erbar2018}. The mobility is therefore a modelling choice that determines which probability evolution is represented as a gradient flow under the resulting transport geometry. 

Using the logarithmic-mean mobility, \citet{Maas2011} defines a dynamic
transport distance \(W_K\) on \(\mathcal P(\X)\) by minimizing a discrete
kinetic action over paths connecting two probability densities. Specifically,
for \(\rho_0,\rho_1\in\mathcal P(\X)\), consider piecewise continuously
differentiable curves
\(\rho:[0,1]\to\mathcal P(\X)\) and measurable curves
\(\psi:[0,1]\to\mathbb R^N\) satisfying, for almost every
\(t\in[0,1]\), the discrete continuity equation
\begin{equation}
\label{eq:discrete-continuity-equation}
\dot\rho_t(x)
+
\sum_{y\in\X}
\bigl(\psi_t(y)-\psi_t(x)\bigr)
K(x,y)
m_{\rho_t}(x,y)
=
0,
\qquad x\in\X,
\end{equation}
with boundary conditions
\(\rho_{t=0}=\rho_0\) and \(\rho_{t=1}=\rho_1\).
The squared transport distance is then defined as the minimum action over
all such admissible pairs:
\begin{equation}
\label{eq:discrete-transport-distance}
\begin{aligned}
W_K(\rho_0,\rho_1)^2
:=
\inf_{\rho,\psi}
\Bigg\{
\frac{1}{2}
\int_0^1
\sum_{x,y\in\X}
&\bigl(\psi_t(x)-\psi_t(y)\bigr)^2
K(x,y)
m_{\rho_t}(x,y)
\pi(x)
\,\mathrm{d}t
\Bigg\}.
\end{aligned}
\end{equation}

The quantity in \eqref{eq:discrete-transport-distance} is the discrete
counterpart of the kinetic action in the Benamou--Brenier formulation.
Accordingly, \(W_K(\rho_0,\rho_1)^2\) is the minimum kinetic energy required
to transport \(\rho_0\) into \(\rho_1\), subject to conservation of mass
through \eqref{eq:discrete-continuity-equation}. The difference
\(\psi_t(x)-\psi_t(y)\) plays the role of a velocity along the edge
connecting \(x\) and \(y\), while \(K(x,y)m_{\rho_t}(x,y)\pi(x)\)
determines the corresponding density-dependent transport weight. ~\cite{Maas2011} shows that \(W_K\) defines a transport metric on
\(\mathcal P(\X)\) and induces a Riemannian structure on its interior
\(\mathcal P_*(\X)\). Crucially, when the mobility \(m_\rho\) is chosen to
be the logarithmic mean, the discrete heat
\eqref{eq:discrete-heat-equation} is the gradient flow, with respect to
\(W_K\), of the relative entropy
\begin{equation}
\label{eq:relative-entropy-discrete}
\mathcal H(\rho)
:=
\sum_{x\in\X}
\rho(x)\log\rho(x)\,\pi(x),
\end{equation}
defined relative to the invariant measure \(\pi\), with the convention
\(0\log 0=0\). 

%% file: sections/method.tex
a\section{Discrete Drifting on Finite Spaces}
\label{sec:discrete-drifting}

Section~\ref{sec:background} reviewed the discrete Wasserstein geometry of
\citet{Maas2011}, under which the heat equation
(\eqref{eq:discrete-heat-equation}) is the \(W_K\)-gradient flow of the relative
entropy
\(\mathcal H(\rho_q)=\mathrm{KL}(q\,\|\,\pi)\).
Consequently, the heat flow relaxes the model distribution \(q\) toward the
invariant measure \(\pi\). For generative modelling, however, the desired
equilibrium is the data distribution \(p\), which need not coincide with
\(\pi\). We therefore replace the entropy relative to \(\pi\) by the
target-relative functional
\begin{equation}
\label{eq:target-relative-entropy}
    \mathcal F_p(\rho_q)
    :=
    \sum_{x\in\X}
        \rho_q(x)
        \log\frac{\rho_q(x)}{\rho_p(x)}
        \,\pi(x)
    =
    \mathrm{KL}(q\,\|\,p),
\end{equation}
where \(\rho_p:=p/\pi\) and \(\rho_q:=q/\pi\). Throughout this section we
assume \(\rho_p,\rho_q\in\mathcal P_*(\X)\), so that the logarithmic density
ratios are well defined.  
This is an entropy-plus-potential free energy with
\(V_p=-\log\rho_p\), so the finite-state gradient calculus of
\citet{Maas2011} applies directly. The corresponding \(W_K\)-gradient flow
is the discrete continuity equation with potential
\(\psi_{p,q} = \log(\rho_p/\rho_q)\). Equivalently, defining the edge current
\begin{equation}
\label{eq:conductance-and-current}
    J_{p,q}(x,y)
    :=
    \pi(x)K(x,y)m_{\rho_q}(x,y)
    \left(
        \log\frac{\rho_q(x)}{\rho_p(x)}
        -
        \log\frac{\rho_q(y)}{\rho_p(y)}
    \right),
\end{equation}
the induced evolution of probability masses is
\begin{equation}
\label{eq:target-kl-gradient-flow}
    \dot q_t(x)
    =
    -\sum_{y\in\X}J_{p,q_t}(x,y).
\end{equation}
By reversibility of \(K\) and symmetry of the logarithmic mean,
\(J_{p,q}(x,y)=-J_{p,q}(y,x)\), so the flow conserves total probability.

The standard gradient-flow dissipation identity specializes here to
\begin{equation}
\label{eq:kl-dissipation}
    \frac{\mathrm d}{\mathrm dt}
    \mathrm{KL}(q_t\,\|\,p)
    =
    -\mathcal I_K(q_t\,\|\,p),
\end{equation}
where
\begin{equation}
\label{eq:relative-fisher-information}
    \mathcal I_K(q\,\|\,p)
    :=
    \frac{1}{2}
    \sum_{x,y\in\X}
        \pi(x)K(x,y)m_{\rho_q}(x,y)
        \left(
            \log\frac{\rho_q(x)}{\rho_p(x)}
            -
            \log\frac{\rho_q(y)}{\rho_p(y)}
        \right)^2 .
\end{equation}
Since \(K\) is irreducible and \(\rho_p,\rho_q>0\),
\(\mathcal I_K(q\,\|\,p)=0\) only when
\(\log(\rho_q/\rho_p)\) is constant on \(\X\); normalization then implies
\(q=p\). Thus \(p\) is the unique stationary point of
\eqref{eq:target-kl-gradient-flow} in \(\mathcal P_*(\X)\). These are direct consequences of the Maas
gradient-flow framework; we use \eqref{eq:target-kl-gradient-flow} as the
distribution-level evolution to be realized by discrete drifting.

\subsection{A Normalization-Free Drift in the Logit Chart}
\label{sec:logit-chart-drift}

The gradient flow in \eqref{eq:target-kl-gradient-flow} cannot be used as a training
signal in the form given: it is written in terms of the normalized laws \(q\) and
\(p\), whereas a parametric model emits unnormalized logits and the target is
typically known only up to a constant. We now show that the flow admits an exact
representation involving only \emph{local differences} of unnormalized
log-densities, and is therefore free of both normalizing constants.

Let \(g:\X\to\mathbb R\) be a log-density parameter and set
\begin{equation}
\label{eq:pi-logit-parameterization}
    q_g(x)
    :=
    \frac{\pi(x)e^{g(x)}}{Z_\pi(g)},
    \qquad
    Z_\pi(g)
    :=
    \sum_{z\in\X}\pi(z)e^{g(z)},
    \qquad
    \rho_g(x)
    :=
    \frac{q_g(x)}{\pi(x)}
    =
    \frac{e^{g(x)}}{Z_\pi(g)} .
\end{equation}
For an oriented edge \((x,y)\), define the model log-density
difference and the target edge score
\begin{equation}
\label{eq:model-and-target-edge-scores}
    \delta_g(x,y)
    :=
    g(y)-g(x),
    \qquad
    s_p(x,y)
    :=
    \log\rho_p(y)-\log\rho_p(x) .
\end{equation}
Both depend only on local differences and are therefore independent of the
corresponding normalization constants; in particular, the target need only
be known up to a multiplicative constant. To express the logarithmic-mean
mobility in terms of these differences, define
\begin{equation} 
\label{eq:logmean-ratio-function}
L(r) := m(1,e^r) = \begin{cases} \displaystyle\frac{e^r-1}{r}, & r\neq0,\\[0.8em] 1, & r=0. 
\end{cases} 
\end{equation}

Transporting the flow of \eqref{eq:target-kl-gradient-flow} to this
chart yields a drift that depends
only on edge quantities. Define the \emph{normalization-free logit drift}
\begin{equation}
\label{eq:normalization-free-logit-drift}
    \Delta_g(x)
    :=
    \sum_{y\in\X}
        K(x,y)
        L\bigl(\delta_g(x,y)\bigr)
        \bigl(
            \delta_g(x,y)-s_p(x,y)
        \bigr),
\end{equation}
a \(K\)-average over the neighbours of \(x\) of the model-minus-target edge
discrepancy \(\delta_g-s_p\), weighted by the logarithmic-mean factor \(L\).
Every quantity on the right-hand side is computable from logits and from the
target up to normalization.

\begin{proposition}[Logit representation of the flow]
\label{prop:normalization-free-flow}
Let \(I\subseteq\mathbb R\) be an interval, let
\(t\mapsto g_t\in\mathbb R^\X\) be continuously differentiable, and set
\(q_t:=q_{g_t}.\) Then \(q_t\) solves the probability-coordinate gradient flow in~\eqref{eq:target-kl-gradient-flow} if and only if there exists a continuous
scalar function \(c:I\to\mathbb R\) such that
\begin{equation}
\label{eq:gauge-covariant-logit-flow}
    \dot g_t
    =
    \Delta_{g_t}
    +
    c_t\mathbf 1,
\end{equation}
where \(\mathbf 1(x)=1\) for every \(x\in\X\).
\end{proposition}
The proof is given in Appendix~\ref{app:discrete-drifting-proofs}.

The drift \eqref{eq:normalization-free-logit-drift} splits into two terms with
distinct roles,
\begin{equation}
\label{eq:drift-attraction-repulsion-split}
    \Delta_g(x)
    =
    \underbrace{
        \sum_{y\in\X} K(x,y) L\bigl(\delta_g(x,y)\bigr)\,\delta_g(x,y)
    }_{\text{model repulsion}}
    \;-\;
    \underbrace{
        \sum_{y\in\X} K(x,y) L\bigl(\delta_g(x,y)\bigr)\,s_p(x,y)
    }_{\text{target attraction}} .
\end{equation}
The first
term depends only on the current model through \(\delta_g\); where the model
already favours \(y\) (\(\delta_g(x,y)>0\)), it raises \(g(x)\), opposing further
concentration there. The second term is the target attraction: where the target favours a neighbour
\(y\) (\(s_p(x,y)>0\)), it lowers \(g(x)\), moving mass toward \(y\).  This is the finite-space analogue of the
attraction--repulsion structure of Euclidean drifting.
\subsubsection{Stop-gradient realization}
\label{sec:discrete-stop-gradient}

Proposition~\ref{prop:normalization-free-flow} determines the logit dynamics
up to an additive gauge. Since \( q_{g+c\mathbf 1} = q_g \) for every scalar \(c\), all gauge choices induce the same probability law.
We therefore choose the zero-gauge representative
\[
    \dot g_t
    =
    \Delta_{g_t}.
\]
Its explicit-Euler discretization with step size \(\eta>0\) is
\begin{equation}
\label{eq:logit-euler-update}
    g^+
    =
    g+\eta\Delta_g.
\end{equation}
The next proposition shows that this chart update is first-order consistent
with the probability-coordinate gradient flow and reproduces its
instantaneous relative-entropy dissipation to first order.

\begin{proposition}[Local consistency and relative-entropy descent]
\label{prop:chart-euler-consistency}
For fixed \(g\), as \(\eta\to0\),
\begin{equation}
\label{eq:chart-law-local-consistency}
    q_{g+\eta\Delta_g}
    =
    q_g+\eta\,q_g\odot\Delta_g
    +
    O(\eta^2),
\end{equation}
and
\begin{equation}
\label{eq:chart-kl-local-descent}
    \mathrm{KL}
    \bigl(
        q_{g+\eta\Delta_g}
        \,\|\,p
    \bigr)
    =
    \mathrm{KL}(q_g\,\|\,p)
    -
    \eta\,\mathcal I_K(q_g\,\|\,p)
    +
    O(\eta^2).
\end{equation}
In particular, whenever \(q_g\neq p\), the update strictly decreases
\(\mathrm{KL}(q_g\,\|\,p)\) for every sufficiently small positive
\(\eta\).
\end{proposition}

The proof is given in Appendix~\ref{app:discrete-drifting-proofs}. The update \eqref{eq:logit-euler-update} is realized by regression onto a frozen
target. Writing \(\bar g\) for the current iterate, define
\begin{equation}
\label{eq:full-batch-stop-gradient-loss}
    \mathcal L_{\mathrm{sg}}(g;\bar g)
    :=
    \frac{1}{2}
    \sum_{x\in\X}
    \Bigl(
        g(x)
        -
        \operatorname{sg}
        \bigl[
            \bar g(x)+\eta\Delta_{\bar g}(x)
        \bigr]
    \Bigr)^2 .
\end{equation}
As in Euclidean drifting, the stop-gradient is what makes this a discretization of the flow rather than an
ordinary loss. Propagating gradients through the target, with the same variable
\(g\) appearing on both sides, collapses the objective to
\begin{equation}
\label{eq:no-stop-gradient-objective}
    \mathcal L_{\mathrm{no\text{-}sg}}(g)
    :=
    \frac{1}{2}
    \bigl\|
        g-\bigl(g+\eta\Delta_g\bigr)
    \bigr\|_2^2
    =
    \frac{\eta^2}{2}
    \bigl\|\Delta_g\bigr\|_2^2 .
\end{equation}
Gradient descent on \(\mathcal L_{\mathrm{no\text{-}sg}}\) minimizes the squared
magnitude of the drift through derivatives of \(\Delta_g\); it does not
implement the vector field \(\Delta_g\) itself.

\subsection{A Minimal Particle Realization}
\label{sec:minimal-particle-lift}

The flow in~\eqref{eq:target-kl-gradient-flow} is Eulerian: it specifies the
net change of probability mass at each state, but does not prescribe how
individual particles should move. On a finite state space, the natural
particle analogue of a Euclidean displacement is a Markov jump. Our goal is
therefore to construct nonnegative jump rates whose directional probability
traffic reproduces the signed current \(J_{p,q}\).

For strictly positive probability mass functions \(p,q\) on \(\X\), define
\[
    \delta_q(x,y)
    :=
    \log\rho_q(y)-\log\rho_q(x),
    \qquad
    \Delta_{p,q}(x)
    :=
    -\frac{1}{q(x)}
    \sum_{y\in\X}J_{p,q}(x,y).
\]
Using the positive \(1\)-homogeneity of the logarithmic mean, the current
in~\eqref{eq:conductance-and-current} admits the equivalent representation
\begin{equation}
\label{eq:law-level-current}
    J_{p,q}(x,y)
    =
    q(x)K(x,y)
    L\bigl(\delta_q(x,y)\bigr)
    \bigl[s_p(x,y)-\delta_q(x,y)\bigr].
\end{equation}
Since jump rates must be nonnegative, we assign to each oriented edge the
positive directional part of this current,
\begin{equation}
\label{eq:minimal-jump-rate}
    a_{p,q}(x,y)
    :=
    \frac{[J_{p,q}(x,y)]_+}{q(x)}
    =
    K(x,y)L\bigl(\delta_q(x,y)\bigr)
    \bigl[s_p(x,y)-\delta_q(x,y)\bigr]_+,
    \qquad x\neq y .
\end{equation}
These rates define the generator
\begin{equation}
\label{eq:particle-generator}
    A_{p,q}(x,y)
    =
    \begin{cases}
        a_{p,q}(x,y), & x\neq y,\\[0.3em]
        -\displaystyle\sum_{z\neq x}a_{p,q}(x,z), & x=y.
    \end{cases}
\end{equation}
Its off-diagonal entries are nonnegative and its rows sum to zero, so for
fixed \(q\) it is a continuous-time Markov generator. Moreover,
antisymmetry of \(J_{p,q}\) yields
\begin{equation}
\label{eq:particle-generator-law-vector}
    (qA_{p,q})(x)
    =
    -\sum_{y\in\X}J_{p,q}(x,y)
    =
    q(x)\Delta_{p,q}(x).
\end{equation}
Thus, when the rates are recomputed from the evolving law, the nonlinear
Kolmogorov equation
\[
    \dot q_t=q_tA_{p,q_t}
\]
coincides exactly with the relative-KL gradient flow
in~\eqref{eq:target-kl-gradient-flow}. Markov-jump realizations of
graph-based probability flows are standard
\citep{Chow2011,sun2023discrete}; the construction
in~\eqref{eq:minimal-jump-rate} is the positive-part realization of the
Maas current.

Discrete drifting requires only a single step of this dynamics. We therefore
freeze the rates at the current law \(q\) and take a forward-Euler step,
\begin{equation}
\label{eq:particle-transition-kernel}
    P_{\eta,p,q}
    :=
    I+\eta A_{p,q}.
\end{equation}
Under the step-size condition
\begin{equation}
\label{eq:particle-cfl-condition}
    \eta
    \max_{x\in\X}
    \sum_{y\neq x}a_{p,q}(x,y)
    \leq 1,
\end{equation}
\(P_{\eta,p,q}\) is a Markov transition kernel, and
\begin{equation}
\label{eq:particle-law-euler-identity}
    qP_{\eta,p,q}
    =
    q+\eta\,q\odot\Delta_{p,q}.
\end{equation}
Hence one frozen stay-or-jump transition realizes exactly the forward-Euler
step of the law-level gradient flow. These generator identities and the
Markov-kernel property are verified in
Lemma~\ref{lem:positive-part-generator} of
Appendix~\ref{app:discrete-drifting-proofs}.

We now relate this particle update to the logit update. When \(q=q_g\), write
\[
    P_{\eta,g}
    :=
    P_{\eta,p,q_g}.
\]
We are now ready to state our next result.
\begin{proposition}[Particle realization of the logit update]
\label{prop:particle-lift}
Let \(q=q_g\), and suppose that
\eqref{eq:particle-cfl-condition} holds with \(q=q_g\). Then
\(P_{\eta,g}\) is a Markov transition kernel and, as \(\eta\to0\),
\begin{equation}
\label{eq:particle-chart-consistency}
    q_{g+\eta\Delta_g}
    =
    q_gP_{\eta,g}
    +
    O(\eta^2).
\end{equation}
Thus one frozen stay-or-jump transition realizes, to first order, the same
evolution of the model law as the logit update
\(g\mapsto g+\eta\Delta_g\).
\end{proposition}

The proof is given in
Appendix~\ref{app:discrete-drifting-proofs}.

Finally, the positive-part construction is \emph{edgewise minimal}. Any
other pair of nonnegative directional rates realizing the same signed
current differs from~\eqref{eq:minimal-jump-rate} by equal counterflow in
the two directions. This counterflow cancels from the marginal evolution
while increasing the total particle traffic. Consequently,
\eqref{eq:minimal-jump-rate} is the unique realization minimizing
directional traffic on each edge. Moreover, when \(q=p\), the current
vanishes and all minimal jump rates are zero, so particles stop pathwise.
These statements are formalized in
Proposition~\ref{prop:edgewise-minimality} and
Corollary~\ref{cor:pathwise-stationarity-at-target} of
Appendix~\ref{app:discrete-drifting-proofs}.

%% file: sections/experiments.tex
\section{Controlled Validation}
\label{sec:experiments}

\begin{figure}[t] \centering \includegraphics[width=\linewidth]{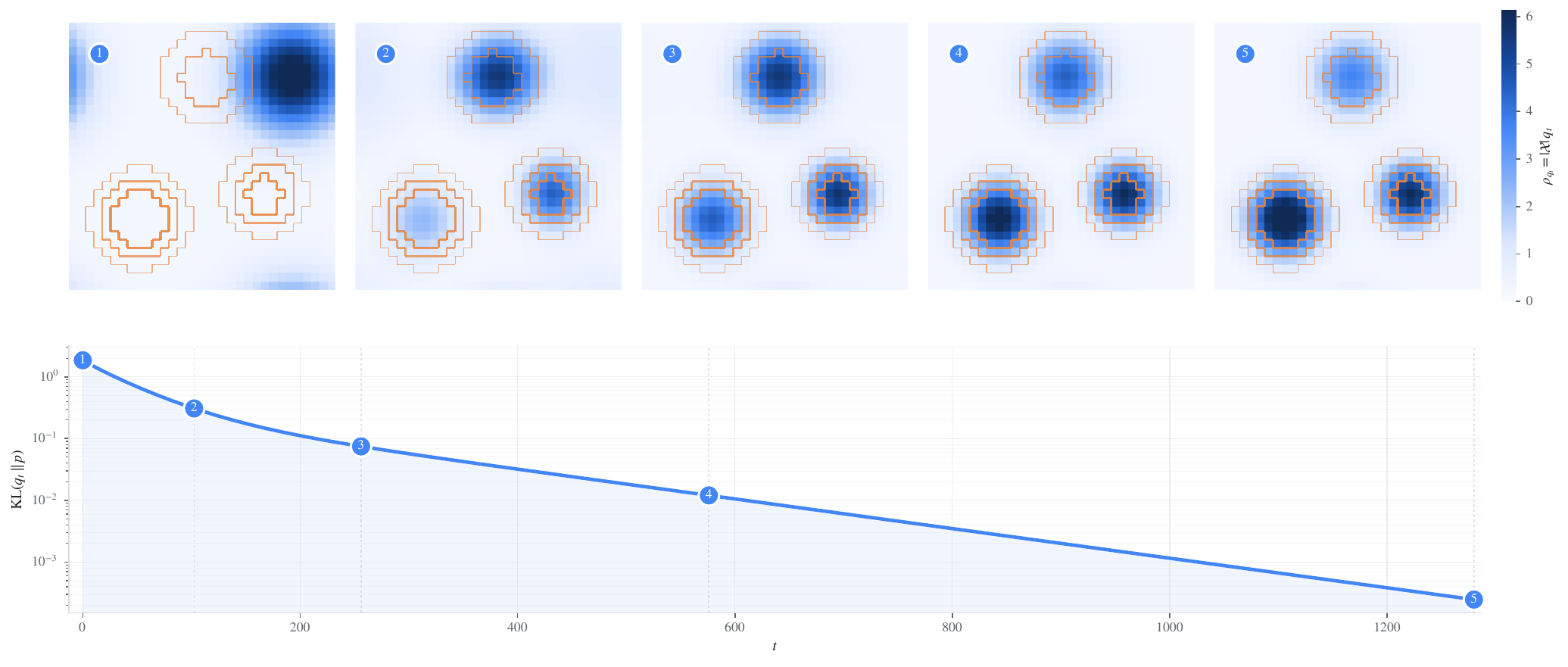} \caption{\textbf{Exact tabular flow on the \(32\times32\) torus}. First 1600 updates. \emph{Top:} Relative density \(\rho_{q_t}=|\X|q_t\) (blue) at five stages; orange contours show the fixed target density. The model begins as a single bump separated from the target, transports mass across low-density regions, and converges to the three modes in the correct proportions. \emph{Bottom:} \(\mathrm{KL}(q_t\,\|\,p)\) versus flow time on a logarithmic scale. Numbered markers associate each displayed distribution with its position along the trajectory.} \label{fig:synthetic-flow} 
\end{figure}

We design a two-stage synthetic experiment to separate the correctness of the
discrete transport from the error introduced by neural amortization. Both
stages use the same target distribution, transport geometry, initialization,
and time discretization. In the first stage, we represent the model distribution
exactly with one logit per state, eliminating estimation, representation, and
optimization error and allowing us to study the proposed transport update in
isolation. In the second stage, we replace this tabular representation with a
shared neural generator and examine how accurately the resulting transport
steps can be amortized. This controlled progression separates discrepancies in
the underlying discrete flow from those introduced by neural approximation. 

The state space is the periodic grid
\( \X=(\mathbb Z/32\mathbb Z)^2\), \(|\X|=1024\), 
equipped with the nearest-neighbour random-walk kernel \(K\) and its uniform
invariant measure \(\pi\). The target distribution \(p\) is a mixture of three
Gaussian bumps with unequal weights and a small uniform floor, ensuring the
strict positivity assumed in Section~\ref{sec:discrete-drifting}. We initialize
the model as a single bump separated from all three target modes, so reaching
\(p\) requires transporting probability across low-density regions rather than
merely performing local relaxation. Unless stated otherwise, both stages use
the same step size \(\eta=0.8\). Full target parameters and remaining
hyperparameters are given in Appendix~\ref{app:experimental-details}.

\paragraph{Exact Tabular Validation.} We first ask whether the proposed update reproduces the predicted discrete
Wasserstein gradient flow when every distributional quantity is available
exactly. At scale, this update is obscured by three additional sources of
error: the relevant density ratios must be estimated from samples, the
generator has limited representational capacity, and optimization is
stochastic. We remove all three by representing the model with one logit per
state, \( q_g=\operatorname{softmax}(g). \)
Because \(\pi\) is uniform, these logits coincide with the chart in
\eqref{eq:pi-logit-parameterization} and can represent any strictly positive
distribution on \(\X\). Since both \(p\) and \(q_g\) are known exactly, all
target scores, currents, mobilities, and jump rates can likewise be evaluated
without estimation. Moreover, the update depends only on model-logit
differences and target log-density ratios along edges, so no normalizing
constant is required. Any discrepancy from the predictions can
therefore be attributed to the transport update or its numerical
discretization rather than to statistical estimation or neural amortization.

We apply the exact tabular update for \(6400\) steps, evaluating every quantity entering the update exactly. Each step is implemented through the stop-gradient objective (\eqref{eq:full-batch-stop-gradient-loss}); with unit learning rate, one gradient step coincides with the explicit Euler update of \eqref{eq:logit-euler-update}. Figure~\ref{fig:synthetic-flow} visualizes the first \(1600\) steps of this trajectory. During this initial phase, the model transports mass across the torus, splits toward the three target regions, and recovers their unequal mixture structure, while \(\mathrm{KL}(q_t\,\|\,p)\) decreases sharply. Continuing the same exact trajectory to \(6400\) steps drives the KL to a numerical floor near \(1.3\times10^{-10}\), while total variation decreases from \(0.71\) to \(7.7\times10^{-6}\). The complete trajectory and additional numerical controls are reported in Appendix~\ref{app:additional-results}. Figure~\ref{fig:synthetic-current} shows the edgewise transport driving this evolution at initialization. The target-dependent component directs mass toward states favoured by \(p\), while the model-dependent component spreads mass away from regions where \(q\) is concentrated. Their sum gives the net current of the target-relative Maas flow, carrying probability from the
initial model surplus toward the target modes.

\begin{figure}[t] \centering \includegraphics[width=\linewidth]{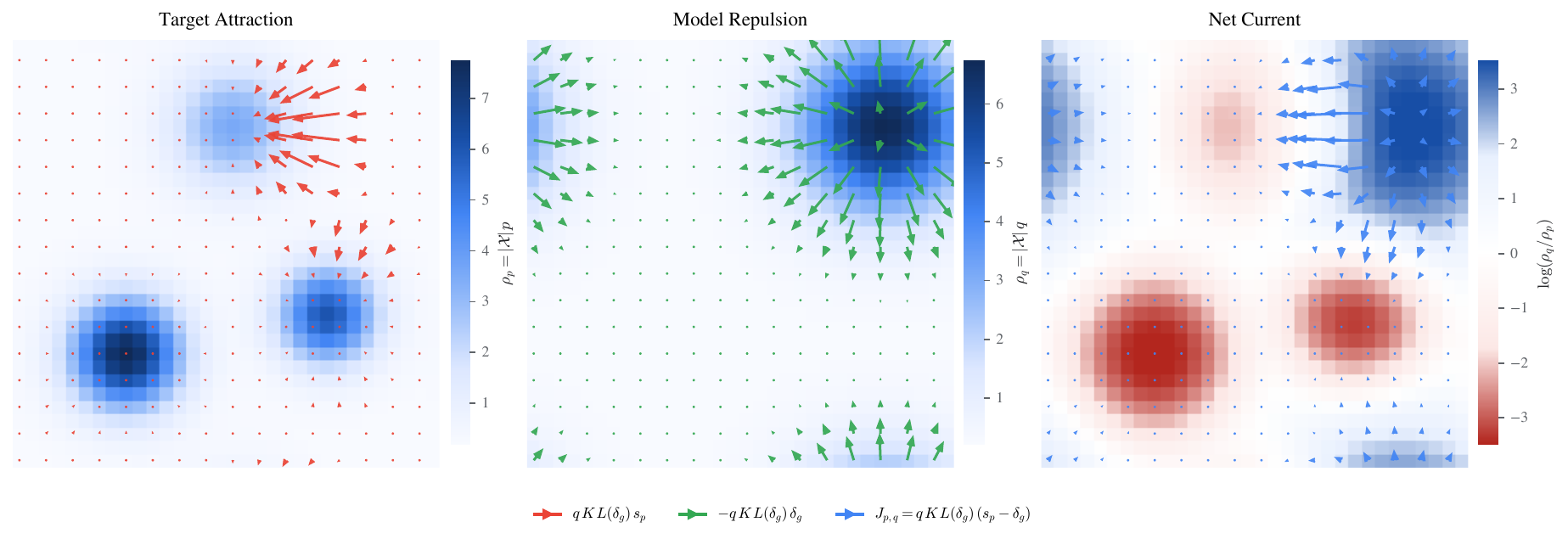} \caption{\textbf{Edgewise decomposition of the relative-entropy current at initialization}. \emph{Left:} Target-attraction current over the target density. \emph{Center:} Model-repulsion current over the model density; by the logarithmic-mean identity, this component coincides with the linear diffusion current. \emph{Right:} Net current over \(\log(\rho_q/\rho_p)\), transporting mass from model surplus toward target deficit.} \label{fig:synthetic-current} 
\end{figure}

\paragraph{Particle Approximation Validation.}
\label{sec:exp-part}
The aforementioned tabular update of \eqref{eq:logit-euler-update} acts on all $|\X|$ states at once, which is possible only because the space is small enough
to enumerate, whereas the particle realization of
Section~\ref{sec:minimal-particle-lift} instead moves individual samples; we therefore ask whether one frozen stay-or-jump transition of \eqref{eq:particle-transition-kernel} reproduces it, as Proposition~\ref{prop:particle-lift} claims. Freezing the minimal rates of \eqref{eq:minimal-jump-rate} at states of the tabular trajectory, we sweep the step size across five decades, from well below the bound of \eqref{eq:particle-cfl-condition} to past it, and evaluate both residuals in
double precision (Figure~\ref{fig:particle-batch}~(left)). The law-level identity of \eqref{eq:particle-law-euler-identity} holds to machine precision throughout, so one transition is not an approximation of the flow of \eqref{eq:target-kl-gradient-flow} but its exact explicit-Euler step. 

The residual of \eqref{eq:particle-chart-consistency} instead decays with a fitted slope of two, so the mismatch between moving particles and moving logits is the
curvature of the chart of \eqref{eq:pi-logit-parameterization} that Proposition~\ref{prop:chart-euler-consistency} already ascribes to the logit update, and nothing beyond it. Second order rather than first is what makes the substitution safe, since the per-step error then accumulates to $O(\eta)$ over a
fixed horizon and vanishes with the step size, and the bound of \eqref{eq:particle-cfl-condition} stays far from binding at the $\eta$ used throughout.

\paragraph{Neural amortization.}
\label{sec:exp-mlp}

We next replace the tabular representation with a finite-capacity neural
generator and ask whether a single parameterized model can absorb the sequence
of transport updates. In the tabular experiment, the logits were themselves
the model parameters, so the Euler update
\eqref{eq:logit-euler-update} could be applied directly. A neural generator,
by contrast, represents its marginal through a mixture of latent-conditional
distributions and cannot be updated by directly perturbing its marginal
logits. We therefore construct training targets using the equivalent frozen
law-level transport of \eqref{eq:particle-transition-kernel}. This reintroduces
representation and optimization error while keeping the model marginal, target
scores, and transport kernel exactly computable, so no estimation error is
present.

The generator is a latent-conditioned MLP
\(G_\theta(\cdot\mid z)\), evaluated on a fixed bank of \(B=256\) latent
codes. The grid, target, initial marginal, transport kernel, and step size are
identical to the tabular experiment. At transport step \(k\), we denote the
conditional distributions by
\[
    R_{k,i}
    :=
    G_{\theta_k}(\cdot\mid z_i),
    \qquad i=1,\ldots,B,
\]
and form the generator marginal
\[
    q_k
    :=
    \frac{1}{B}\sum_{i=1}^B R_{k,i}.
\]
The current marginal determines a single frozen transport kernel, which is
then applied to every conditional:
\begin{equation}
    T_{k,i}
    :=
    R_{k,i}P_{\eta,p,q_k},
    \qquad
    P_{\eta,p,q_k}
    :=
    I+\eta A_{p,q_k}.
    \label{eq:mlp-teacher}
\end{equation}
Holding the teachers \(T_{k,i}\) fixed, we take a small number of stochastic
gradient steps on their conditional cross-entropy with
\(G_\theta(\cdot\mid z_i)\), corresponding to the stop-gradient objective in
\eqref{eq:full-batch-stop-gradient-loss}, and then recompute the marginal and
the next transport step. We repeat this procedure for \(1600\) transport
updates.

This construction preserves the desired marginal dynamics. Indeed, linearity
of the frozen kernel gives
\begin{equation}
    \frac{1}{B}\sum_{i=1}^B T_{k,i}
    =
    \left(\frac{1}{B}\sum_{i=1}^B R_{k,i}\right)
    P_{\eta,p,q_k}
    =
    q_kP_{\eta,p,q_k}.
\end{equation}
Thus, if the conditional teachers were fitted exactly, the neural marginal
would advance by exactly the frozen law-level Euler step
\eqref{eq:particle-law-euler-identity}, which agrees with the corresponding
logit update to first order in \(\eta\). In practice, the same network must
successively approximate every teacher along the trajectory, so fitting errors
may accumulate. We therefore measure its deviation from the exact transported
trajectory throughout the rollout rather than only at the final iterate.

\begin{figure}[t]
\centering
\begin{minipage}[c]{0.49\linewidth}
  \centering
  \includegraphics[width=\linewidth]{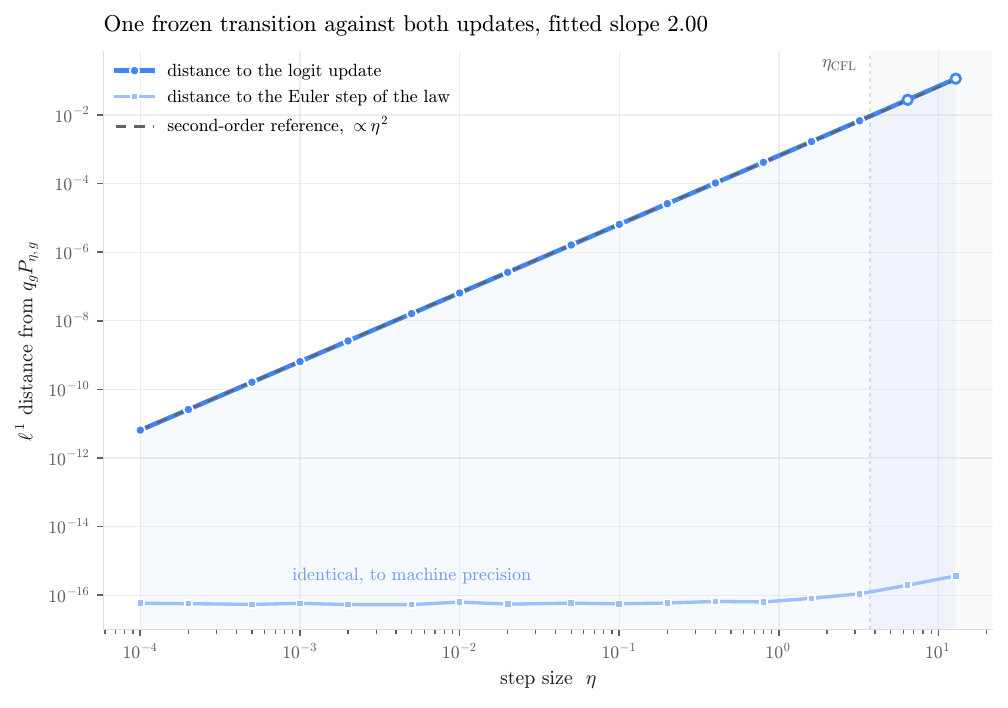}
\end{minipage}\hfill
\begin{minipage}[c]{0.49\linewidth}
  \centering
  \includegraphics[width=\linewidth]{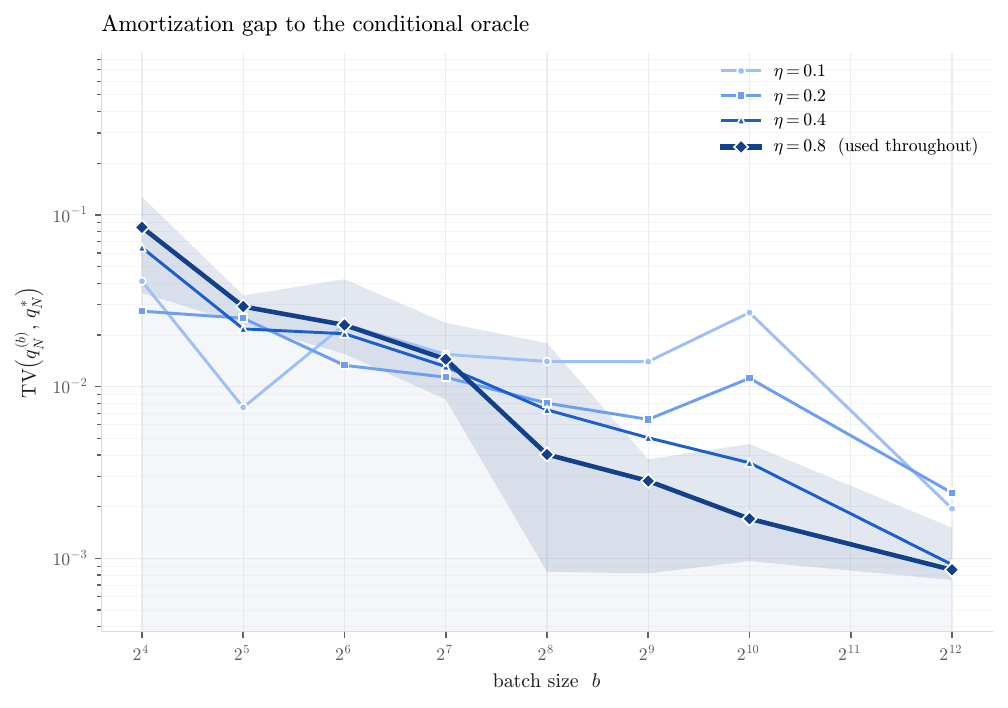}
\end{minipage}
\caption{\emph{Left:} one frozen transition of
\eqref{eq:particle-transition-kernel} against the two updates it replaces: fitted slope $2.00$ for the residual of \eqref{eq:particle-chart-consistency},
and machine precision for the Euler step of the law, which \eqref{eq:particle-law-euler-identity} makes an identity. The vertical rule marks $\eta_{\mathrm{CFL}}$, the largest step size admitted by \eqref{eq:particle-cfl-condition}, past which no particle can make the move. \emph{Right:} final gap to the
conditional oracle against batch size $b$, at four step sizes with $\eta N$
fixed; medians over five seeds, band spanning seeds at $\eta=0.8$. }
\label{fig:particle-batch}
\end{figure}

As a reference, we use a conditional oracle that starts from the same
conditionals and applies \eqref{eq:mlp-teacher} exactly at every transport
step, without neural fitting. Figure~\ref{fig:mlp-final} shows that the neural
generator closely follows this oracle: the initial mass is transported across
the low-density regions and distributed among the three target modes with the
correct unequal weights, while the final marginal is visually
indistinguishable from the target at the resolution shown. The fitted
conditionals attain comparable error on training and held-out latent codes,
indicating that the network learns the transport rule rather than merely
memorizing the training bank. The right panel reports the total-variation gap
between the neural rollout and the oracle together with the oracle's distance
to \(p\). Increasing the per-step latent batch size \(b\) reduces the
neural--oracle gap by more than an order of magnitude, with the smallest error
obtained at \(b=B\). This behavior indicates that the remaining discrepancy is
primarily associated with fitting each transport teacher rather than with the
underlying transport itself. Complete KL and total-variation results, together
with the architecture, optimizer, and held-out evaluation protocol, are given
in Appendices~\ref{app:experimental-details}
and~\ref{app:additional-results}.

%% file: sections/conclusion.tex
\section{Conclusion}
\label{sec:conclusion}

We introduced Discrete Drifting, which replaces Euclidean drift with a relative-KL gradient flow in Maas geometry and realizes each frozen flow step through Markov jumps. In early results, on an enumerable finite-state problem, the proposed transport follows the predicted KL-descending dynamics and can be amortized by a finite neural generator while preserving one-step inference. Future work will focus on scalable local ratio estimation and evaluation on real discrete data against strong discrete diffusion and one-step baselines.




%% file: sections/appendix.tex
\input{sections/appendix/appendix_algo}

\input{sections/appendix/appendix_experiments}
\clearpage
\newpage
\input{sections/appendix/appendix_results}

\clearpage
\newpage
\input{sections/appendix/appendix_proof}


%% file: sections/appendix/appendix_algo.tex
\section{Neural Amortization of the Drifting Update}
Section~\ref{sec:exp-mlp} replaces the tabular logits with a shared generator
and fits each frozen transition rather than applying it. Two facts make that
substitution well posed, and we establish them here. First, the conditional
teachers average to the exact law-level Euler step, so the generator is asked
to follow the trajectory validated in Section~\ref{sec:experiments} rather than
a nearby one. Second, the distance between the fitted marginal and that step is
controlled by the mean conditional fitting error, which separates the exact
teacher construction from the amortization gap we measure.

\subsection{Conditional teachers}

At outer step \(k\), let
\begin{equation}
\begin{aligned}
    R_{k,z}&:=G_{\theta_k}(\cdot\mid z),\\
    q_k&:=\mathbb E_{z\sim\nu}[R_{k,z}],\\
    T_{k,z}&:=R_{k,z}P_{\eta,p,q_k}.
\end{aligned}
    \label{eq:app-conditional-teacher}
\end{equation}
Every conditional is advanced by the same kernel, built once from th
marginal and then frozen. Linearity gives
\begin{equation}
    \mathbb E_z[T_{k,z}]
    =\mathbb E_z[R_{k,z}]P_{\eta,p,q_k}
    =q_kP_{\eta,p,q_k},
    \label{eq:app-teacher-marginal}
\end{equation}
which by \eqref{eq:particle-law-euler-identity} is exactly the law-l
update. No approximation enters before fitting.

\subsection{Fitting and the amortization gap}

Holding the teachers fixed, we fit the generator by minimizing the mean
conditional cross-entropy
\begin{equation}
    \mathcal L_k(\theta)
    :=-\mathbb E_{z\sim\nu}\sum_{x\in\X}
      T_{k,z}(x)\log G_\theta(x\mid z),
    \label{eq:app-amortized-loss}
\end{equation}
the per-latent form of the stop-gradient objective
\eqref{eq:full-batch-stop-gradient-loss}. Because the teachers are c
\begin{equation}
    \mathcal L_k(\theta)
    =\mathbb E_z\Bigl[
      H(T_{k,z})
      +\mathrm{KL}\bigl(T_{k,z}\,\|\,G_\theta(\cdot\mid z)\bigr)
      \Bigr],
    \label{eq:app-cross-entropy-decomposition}
\end{equation}
so the only \(\theta\)-dependent term is the mean divergence to the teachers.
In the exact-fitting limit \(G_{\theta_{k+1}}(\cdot\mid z)=T_{k,z}\)
learned marginal absorbs one exact Euler step.

Finite fitting leaves a residual, whose effect on the marginal is bounded by
convexity of total variation. Writing
\(\bar q_{k+1}:=\mathbb E_zG_{\theta_{k+1}}(\cdot\mid z)\),
\begin{equation}
    \operatorname{TV}\!\left(\bar q_{k+1},q_kP_{\eta,p,q_k}\right)
    \leq\mathbb E_z\operatorname{TV}\!\left(
      G_{\theta_{k+1}}(\cdot\mid z),T_{k,z}\right).
    \label{eq:app-marginal-amortization-bound}
\end{equation}
The right-hand side is the local fitting error
\eqref{eq:app-training-teacher-error} and the left-hand side the induced
marginal error \eqref{eq:app-marginal-teacher-error}, both reported in
Appendix~\ref{app:additional-neural-results}. The teacher construction
therefore contributes nothing to the amortization gap of
Figure~\ref{fig:particle-batch}~(right), and every departure from the validated
trajectory is attributable to the fitting budget.

\subsection{Training loop and cost}

For \(|E|\) directed edges and \(B\) latents, one outer step costs \(O(|E|)\)
to build the marginal rates and \(O(B|E|)\) to transport all conditionals, so
teacher construction is linear in the bank size and independent of the inner
budget. The Markov kernel is used only during training: sampling from the
fitted generator is a single forward pass.

\begin{algorithm}[!ht]
\caption{Amortized discrete drifting. One outer step advances every conditional
by the same frozen kernel; the inner loop fits the generator to the result.}
\label{alg:amortized-drifting}

\begin{algorithmic}[1]
\Require Target scores \(s_p\), latent law \(\nu\), step size \(\eta\),
         outer steps \(N\), inner steps \(m\)

\State Initialize \(\theta_0\)

\For{\(k = 0, \ldots, N-1\)}
    \State \(R_{k,z} \gets G_{\theta_k}(\cdot \mid z)\)
        \Comment{Conditionals}

    \State \(q_k \gets \mathbb{E}_{z \sim \nu}[R_{k,z}]\)
        \Comment{Model marginal}

    \State \(a_{p,q_k} \gets\) rates of
        \eqref{eq:minimal-jump-rate} from \(\delta_{q_k}\) and \(s_p\)

    \State \(P_{\eta,p,q_k} \gets I + \eta A_{p,q_k}\)
        \Comment{Markov under \eqref{eq:particle-cfl-condition}}

    \State \(T_{k,z} \gets
        \operatorname{sg}\bigl[R_{k,z}P_{\eta,p,q_k}\bigr]\)
        \Comment{Frozen teachers}

    \State \(\theta_{k+1} \gets m\) optimizer steps on
        \(\mathcal{L}_k\) of \eqref{eq:app-amortized-loss},
        starting from \(\theta_k\)
\EndFor

\Ensure \(\theta_N\)
\end{algorithmic}
\end{algorithm}

%% file: sections/appendix/appendix_experiments.tex
\clearpage
\newpage
\section{Experimental Setup and Reproducibility Details}
\label{app:experimental-details}

This section provides the complete setup for the controlled experiments in
Section~\ref{sec:experiments}. Both experiments use the same enumerable state
space, target distribution, initialization, transport graph, and time
discretization. The tabular experiment evaluates the proposed transport without
representation or optimization error. The neural experiment then introduces
finite representation and fitting while retaining exact, enumerable transport
teachers.

\subsection{Shared enumerable problem}
\label{app:shared-problem}

\paragraph{State space and transport graph.}
The state space is the periodic grid
\(\mathcal X=\mathbb Z_{32}\times\mathbb Z_{32}\), represented by
\(x=(x_1,x_2)\in\{0,\ldots,31\}^2\), with
\(|\mathcal X|=1024\). Each state is connected to its four nearest neighbours,

$$
    \mathcal N(x)
    =\{x\pm(1,0),x\pm(0,1)\}\pmod{32}.
$$

We use the nearest-neighbour random-walk kernel
\begin{equation}
K(x,y)=
\begin{cases}
1/4, & y\in\mathcal N(x),\\
0,   & \text{otherwise},
\end{cases}
\qquad
\pi(x)=\frac{1}{1024}.
\label{eq:app-torus-kernel}
\end{equation}
This kernel is reversible with respect to the uniform reference law \(\pi\).

\paragraph{Periodic Gaussian family.}
For \(a,b\in\{0,\ldots,31\}\), define the periodic distance
\begin{align}
d_{32}(a,b)
&:=\min\{|a-b|,\,32-|a-b|\},\notag\\
d_{\mathbb T}^2(x,\mu)
&:=d_{32}(x_1,\mu_1)^2+d_{32}(x_2,\mu_2)^2.
\end{align}
The normalized periodic Gaussian with centre \(\mu\) and scale \(\sigma\) is
\begin{align}
Z_{\mu,\sigma}
&:=\sum_{y\in\mathcal X}
\exp\left[-\frac{d_{\mathbb T}^2(y,\mu)}{2\sigma^2}\right],\notag\\
\varphi_{\mu,\sigma}(x)
&:=Z_{\mu,\sigma}^{-1}
\exp\left[-\frac{d_{\mathbb T}^2(x,\mu)}{2\sigma^2}\right].
\label{eq:app-periodic-gaussian}
\end{align}

\paragraph{Target and initialization.}
The target is a mixture of three periodic Gaussians with a uniform floor:
\begin{equation}
\begin{aligned}
p(x)
&=(1-\varepsilon)\widetilde p(x)+\varepsilon\pi(x),\\
\widetilde p(x)
&:=\sum_{j=1}^{3}w_j\varphi_{\mu_j,\sigma_j}(x),
\qquad \sum_{j=1}^{3}w_j=1,\\
q_0(x)
&=\varphi_{\mu_0,\sigma_0}(x).
\end{aligned}
\label{eq:app-synthetic-laws}
\end{equation}
The target modes have unequal weights and are separated by low-density regions.
The initial mode is placed away from all three target modes, so reaching the
target requires both splitting and transporting probability mass. The uniform
floor ensures the strict positivity required by the flow. All centres and
scales are given in lattice units, and take the values
\begin{equation}
\begin{aligned}
(\mu_1,\sigma_1,w_1)&=\bigl((8.00,8.00),\,2.88,\,0.462\bigr),\\
(\mu_2,\sigma_2,w_2)&=\bigl((23.04,10.88),\,2.56,\,0.284\bigr),\\
(\mu_3,\sigma_3,w_3)&=\bigl((16.00,24.96),\,3.20,\,0.254\bigr),\\
(\mu_0,\sigma_0)&=\bigl((26.24,24.96),\,4.48\bigr),
\qquad \varepsilon=0.168 .
\end{aligned}
\label{eq:app-synthetic-parameters}
\end{equation}
The floor is set so that the least likely state still carries relative density
\(\rho_p=0.168\), against \(7.73\) at the highest mode. The implementation
builds the same law from unnormalized periodic Gaussians and a log-domain
floor; \eqref{eq:app-synthetic-parameters} reports the equivalent parameters of
the normalized form in \eqref{eq:app-synthetic-laws}.

Because \(\mathcal X\) is finite, every distribution, directed edge, current,
transition probability, and conditional teacher can be evaluated by
enumeration. Using the same problem throughout allows us to distinguish errors
in the proposed flow from those introduced by particle sampling or neural
amortization.

\subsection{Exact tabular validation}
\label{app:tabular-experiment}

The tabular model assigns one unconstrained logit \(g(x)\) to each state and
represents

$$
    q_g=\operatorname{softmax}(g).
$$

The initial logits represent \(q_0\) from
\eqref{eq:app-synthetic-laws}. At each outer step, we enumerate all states and
directed nearest-neighbour edges to compute
\(\delta_g(x,y)\), \(s_p(x,y)\), \(L(\delta_g(x,y))\), the current, the logit
drift, and the frozen Markov transition.

Numerically, the logarithmic-mean factor \(L(r)\) is evaluated as
\(\operatorname{expm1}(r)/r\) away from zero and by its continuous limit at
zero. Writing the frozen transition of \eqref{eq:particle-transition-kernel}
entrywise,
\begin{equation}
P_{\eta,p,q}(x,y)=
\begin{cases}
\eta\,a_{p,q}(x,y), & y\in\mathcal N(x),\\[0.2em]
1-\eta\sum_{z\in\mathcal N(x)}a_{p,q}(x,z), & y=x,\\[0.2em]
0, & \text{otherwise},
\end{cases}
\label{eq:app-particle-kernel-entries}
\end{equation}
we verify that every entry is nonnegative before applying each update. This is
the step-size condition of \eqref{eq:particle-cfl-condition}: the smallest stay
probability over the tabular trajectory is \(0.950\), attained at
\(\eta=0.8\).

\begin{table}[t]
\centering
\small
\begin{tabular}{@{}p{0.31\linewidth}p{0.61\linewidth}@{}}
\toprule
Component & Setting \\
\midrule
Model & One unconstrained logit per state;
\(q_g=\operatorname{softmax}(g)\) \\

Initialization & Exact logits representing \(q_0\) \\

Drift evaluation & All \(1024\) states and directed nearest-neighbour edges \\

Outer updates & \(6400\) \\

Flow step & \(\eta=0.8\) \\

Final flow time & \(5120\) \\

Chart optimization & One full-batch step on \eqref{eq:full-batch-stop-gradient-loss}
with learning rate \(1\) \\

Exact references & Probability-space Euler update and frozen Markov law \\

Particle counts & \(M\in[2\times10^3,2\times10^5]\);
six repetitions per count \\
\bottomrule
\end{tabular}
\caption{Configuration of the exact tabular validation.}
\label{tab:app-tabular-configuration}
\end{table}

To quantitatively validate every step of the construction, from the law-level
flow of \eqref{eq:target-kl-gradient-flow} to the sampled stop-gradient update,
we perform six numerical checks:

\begin{itemize}
\item \textbf{Chart versus density Euler.}
We compare the distribution represented by
\(g+\eta\Delta_g\) with the explicit density update
\(q_g+\eta q_g\odot\Delta_g\).

\item \textbf{Frozen Markov law versus density Euler.}
We compare \(q_gP_{\eta,p,q_g}\) with
\(q_g+\eta q_g\odot\Delta_g\).

\item \textbf{Particle histogram versus frozen Markov law.}
We compare the empirical post-jump histogram
\(\widehat q_M^+\) with \(q_gP_{\eta,p,q_g}\).

\item \textbf{Frozen transition versus logit update.}
We compare \(q_gP_{\eta,p,q_g}\) with \(q_{g+\eta\Delta_g}\), which is the
residual of \eqref{eq:particle-chart-consistency}.

\item \textbf{Stop-gradient control.}
We compare the intended frozen-target regression in
\eqref{eq:full-batch-stop-gradient-loss} with the
differentiate-through-target objective in
\eqref{eq:no-stop-gradient-objective}.

\item \textbf{Sampled state--edge control.}
We compare an unbiased sampled implementation with full enumeration.
States are drawn uniformly with replacement, \(4096\) per update, with two
neighbours sampled per drawn state. The scaled loss is minimized by gradient
descent at learning rate \(0.5\) for \(12800\) updates, so the sampled control
covers the same flow time \(5120\) as the enumerated run.
\end{itemize}

Results for these checks are reported in
Appendix~\ref{app:additional-tabular-results}. Additionally the
frozen-transition order test, also appears in
Figure~\ref{fig:particle-batch}~(left) of the main text and in Table~\ref{tab:app-lift-order} below.

\paragraph{Step-size and particle sweeps.}
The chart-versus-Euler and frozen-transition checks are order tests, so their sweeps are specified here. The
production flow runs in single precision, whose probability floor near
\(10^{-7}\) would hide a second-order signal, and both order tests are
therefore recomputed on the host in double precision from the same frozen
logits. The chart-versus-Euler control freezes the trajectory at outer step
\(1280\) and sweeps
\(\eta\in\{1,2,5\}\times\{10^{-4},10^{-3},10^{-2},10^{-1}\}\cup\{4\times10^{-1}\}\).
The frozen-transition control freezes the trajectory at outer steps
\(0\), \(128\), \(320\), and \(720\), and extends the same grid to
\(\eta\in\{0.8,1.6,3.2,6.4,12.8\}\), which carries it past the step-size
bound at the earliest state. The slope is fitted by least squares in
\(\log\eta\) over the points that satisfy \eqref{eq:particle-cfl-condition}
and whose residual exceeds \(10^{-14}\), so that neither an inadmissible step
nor double roundoff enters the fit. Table~\ref{tab:app-lift-order} reports the
four frozen states; Figure~\ref{fig:particle-batch}~(left) plots the first.
The particle control freezes the trajectory at outer step \(320\), uses
\(\eta=0.8\), and draws
\(M\in\{2,5\}\times\{10^3,10^4\}\cup\{10^5,2\times10^5\}\) particles,
with six independent repetitions at each count.

\begin{table}[t]
\centering
\small
\begin{tabular}{@{}rrrrr@{}}
\toprule
Flow time \(t\) & \(\mathrm{KL}(q_t\,\|\,p)\) & Step-size bound &
Fitted slope & Law-identity residual \\
\midrule
\(0\)     & \(1.85\)             & \(3.72\)  & \(2.0014\) & \(3.7\times10^{-16}\) \\
\(102.4\) & \(3.08\times10^{-1}\) & \(7.19\)  & \(2.0001\) & \(1.2\times10^{-16}\) \\
\(256.0\) & \(7.47\times10^{-2}\) & \(16.06\) & \(2.0000\) & \(9.3\times10^{-17}\) \\
\(576.0\) & \(1.20\times10^{-2}\) & \(38.20\) & \(2.0001\) & \(6.9\times10^{-17}\) \\
\bottomrule
\end{tabular}
\caption{Frozen-transition order test at four states of the tabular
trajectory. The step-size bound is the largest \(\eta\) admitted by
\eqref{eq:particle-cfl-condition}, the fitted slope is that of
\(\|q_{g+\eta\Delta_g}-q_gP_{\eta,g}\|_1\), and the last column is the
largest observed violation of \eqref{eq:particle-law-euler-identity} over the
whole sweep. The bound grows as the flow converges and the rates decay.}
\label{tab:app-lift-order}
\end{table}

\subsection{Finite neural amortization}
\label{app:neural-experiment}

\paragraph{Generator and initialization.}
The neural generator is a latent-conditioned MLP
\(G_\theta(\cdot\mid z)\), where \(z\sim\mathcal N(0,I_{16})\). It has three
hidden layers of width \(128\), SiLU activations, and a \(1024\)-way softmax
output. The headline experiment draws a fixed training bank of \(B=256\)
latents and uses the uniform distribution over this bank as the finite latent
prior. A separate fixed bank of \(256\) held-out latents is never used for
optimization.

A statewise output bias is calibrated once so that the training-bank marginal
matches the tabular initialization \(q_0\), and is then frozen. The bias
therefore supplies only the common initial marginal. All subsequent transport
must be represented by the shared MLP weights. The bias is obtained by a damped
fixed-point iteration: starting from \(\log q_0\) centred to zero mean, each
of \(700\) iterations adds \(0.75\bigl(\log q_0-\log\bar q\bigr)\) to the
current bias, where \(\bar q\) is the bank marginal it induces, and recentres
the result. The iteration count is fixed rather than tolerance-driven, so
calibration is deterministic given the seed.

\paragraph{Exact conditional teachers.}
For the fixed training bank, let
\begin{equation}
\begin{aligned}
R_{k,i}
&=G_{\theta_k}(\cdot\mid z_i),\\
q_k
&=\frac1B\sum_{i=1}^{B}R_{k,i},\\
T_{k,i}
&=R_{k,i}P_{\eta,p,q_k}.
\end{aligned}
\label{eq:app-finite-bank-teacher}
\end{equation}
We enumerate all \(1024\) states and all \(B\) conditionals. Consequently,
\(q_k\), \(P_{\eta,p,q_k}\), and every \(T_{k,i}\) are exact for this finite
latent prior. Moreover,
\begin{equation}
\frac1B\sum_{i=1}^{B}T_{k,i}
=q_kP_{\eta,p,q_k},
\label{eq:app-mean-teacher}
\end{equation}
so the mean teacher is exactly the density-Euler update validated in the
tabular experiment.

At each outer step, all teachers are frozen and fitted jointly by minimizing
their mean conditional cross-entropy. We take \(16\) Adam steps and carry the
optimizer moments across outer updates. Carrying the moments is useful near
convergence, where the required teacher movements become small.

\begin{table}[t]
\centering
\small
\begin{tabular}{@{}p{0.31\linewidth}p{0.61\linewidth}@{}}
\toprule
Component & Setting \\
\midrule
Latent distribution & \(z\sim\mathcal N(0,I_{16})\) \\

Generator & Three hidden layers of width \(128\), with SiLU activations \\

Output & \(1024\)-way softmax \\

Training latent bank & \(B=256\), fixed for the complete run \\

Held-out latent bank & \(256\), fixed and never optimized against \\

Initialization & Statewise output bias calibrated to \(q_0\), then frozen \\

Outer updates & \(1600\) \\

Flow step & \(\eta=0.8\) \\

Final flow time & \(1280\) \\

Inner optimization & \(16\) full-bank Adam steps per outer update \\

Adam learning rate & \(2\times10^{-3}\) \\

Optimizer state & First and second moments carried across outer updates \\

Independent repetitions & Independent deterministic seeds \\
\bottomrule
\end{tabular}
\caption{Configuration of the finite neural amortization experiment.}
\label{tab:app-mlp-configuration}
\end{table}

\paragraph{Conditional oracle.}
We compare the neural rollout with an oracle that applies the same frozen
transition directly to each conditional:
\begin{equation}
R^*_{k+1,i}
=R^*_{k,i}P_{\eta,p,q_k^*},
\qquad
q_k^*=\frac1B\sum_{i=1}^{B}R^*_{k,i}.
\label{eq:app-conditional-oracle}
\end{equation}
This oracle contains no neural fitting error. As an implementation check, its
marginal is also integrated independently using the probability-space Euler
update. Across the rollout, the two implementations differ by at most
\(6.6\times10^{-9}\) in any state's probability.

\subsection{Finite-batch sweep}
\label{app:batch-sweep}

The headline neural experiment enumerates its complete \(256\)-latent training
bank. We separately examine how finite-batch fitting affects the accumulated
amortization error. For this experiment, we fix a population of
\(B_{\mathrm{pop}}=4096\) latents and vary the finite-batch size over

$$
    b\in\{16,32,64,128,256,512,1024,4096\}.
$$

The \(b=4096\) condition uses the complete latent population. Unless stated
otherwise, the architecture, flow step, optimizer, number of inner updates, and
number of outer updates match Table~\ref{tab:app-mlp-configuration}. Each
condition is repeated over independent deterministic seeds.

At each outer update we draw a fresh subset of \(b\) latents without
replacement from the fixed population, and the same subset is reused for all
\(16\) inner steps of that update. The subset alone defines the marginal
\(q_k\), the frozen kernel \(P_{\eta,p,q_k}\), and every conditional
teacher, so a small \(b\) perturbs the update itself and not merely the
gradient. Evaluation is unaffected by this choice: the reported neural marginal
\(q_k^{(b)}\) and the conditional-oracle marginal \(q_k^*\) are always
averaged over the complete \(4096\)-latent population. Each subset is the
length-\(b\) prefix of a permutation drawn from a seed-derived key, so
conditions sharing a seed also share the latent population, the initial
parameters, the optimizer state, and nested batches.

Let \(q_k^{(b)}\) denote the neural marginal produced with finite-batch size
\(b\), and let \(q_k^*\) denote the conditional-oracle marginal defined on the
same latent population. For \(N_{\mathrm{out}}=1600\), the quantities reported
in Figure~\ref{fig:particle-batch}~(right) are
\begin{align}
g_{\mathrm{final}}^{(b)}
&:=
\operatorname{TV}\left(
q_{N_{\mathrm{out}}}^{(b)},
q_{N_{\mathrm{out}}}^*
\right),
\label{eq:app-final-oracle-gap}\\
g_{\mathrm{avg}}^{(b)}
&:=
\frac{1}{N_{\mathrm{out}}+1}
\sum_{k=0}^{N_{\mathrm{out}}}
\operatorname{TV}\left(q_k^{(b)},q_k^*\right).
\label{eq:app-average-oracle-gap}
\end{align}
The first measures endpoint agreement, while the second measures agreement over
the complete rollout. Both marginals are recorded at every outer step, and the
time average runs over all \(N_{\mathrm{out}}+1\) of them, including the
common initialization at \(k=0\).

\subsection{Metrics and statistical reporting}
\label{app:experimental-metrics}

For probability laws \(r\) and \(s\) on \(\mathcal X\), we report
\begin{equation}
\begin{aligned}
\operatorname{TV}(r,s)
&:=\frac12\sum_{x\in\mathcal X}|r(x)-s(x)|,\\
\mathrm{KL}(r\,\|\,s)
&:=\sum_{x\in\mathcal X}
r(x)\log\frac{r(x)}{s(x)}.
\end{aligned}
\label{eq:app-evaluation-metrics}
\end{equation}

For the neural experiment, the local training error is
\begin{equation}
e_{\mathrm{train},k}
:=\frac1B\sum_{i=1}^{B}
\operatorname{TV}(R_{k+1,i},T_{k,i}),
\label{eq:app-training-teacher-error}
\end{equation}
and the induced marginal fitting error is
\begin{equation}
e_{\mathrm{marg},k}
:=\operatorname{TV}\left(q_{k+1},q_kP_{\eta,p,q_k}\right).
\label{eq:app-marginal-teacher-error}
\end{equation}

For the held-out bank, we construct teachers using the same frozen transition
as the training bank:

$$
    T^{\mathrm h}_{k,i}
    :=R^{\mathrm h}_{k,i}P_{\eta,p,q_k}.
$$

The held-out fitting error is
\begin{equation}
e_{\mathrm{held},k}
:=\frac1{B_{\mathrm h}}\sum_{i=1}^{B_{\mathrm h}}
\operatorname{TV}\left(R^{\mathrm h}_{k+1,i},T^{\mathrm h}_{k,i}\right).
\label{eq:app-heldout-teacher-error}
\end{equation}

Results for the headline neural experiment report means across seeds, with
pointwise \(95\%\) \(t\) intervals
\begin{equation}
\bar m\ \pm\ t_{0.975,\,n-1}\frac{s}{\sqrt n},
\label{eq:app-t-confidence-interval}
\end{equation}
where \(n\) is the number of seeds and \(s\) the sample standard deviation
across them. The finite-batch sweep has too few seeds per condition for a
meaningful interval, so we report the median across seeds and the full seed
range instead.

To test whether the finite latent bank materially affects evaluation, each
final generator is additionally reevaluated using fresh Gaussian banks of
sizes \(256\), \(1024\), and \(4096\). The \(4096\)-latent marginal serves as
the reference. This diagnostic changes only the bank used to evaluate a fixed
generator; it is distinct from the finite-batch sweep, which changes the
fitting procedure during the rollout.

%% file: sections/appendix/appendix_results.tex
\clearpage
\newpage
\section{Additional Results}
\label{app:additional-results}

\subsection{Exact transport and particle realization}
\label{app:additional-tabular-results}

The main text establishes that the tabular model transports mass toward the target and that one frozen particle transition is consistent with the corresponding logit update. Here, we test each link in this construction over the
complete trajectory: whether the stop-gradient update follows the intended probability flow, whether its descent obeys the predicted dissipation identity, and whether the particle realization introduces any error beyond Monte Carlo sampling.

\paragraph{The tabular update follows the intended flow.}
Figure~\ref{fig:app-tabular-diagnostics} reports KL and total variation along the stop-gradient chart update, the probability-space Euler reference, the sampled state--edge update, and the differentiate-through-target control. The
stop-gradient chart update closely tracks the Euler reference throughout the descent, with both approaching their numerical floors. Replacing full enumeration with unbiased state--edge sampling preserves the same convergence behaviour, showing that stochastic evaluation does not alter the underlying
dynamics. By contrast, differentiating through the transport target produces only limited descent and remains far from the target. This confirms that the stop-gradient in \eqref{eq:full-batch-stop-gradient-loss} is essential: it implements the prescribed vector field, whereas the differentiated-target objective does not.

\paragraph{The descent obeys the gradient-flow identity.}
Convergence toward \(p\) alone would not establish that the update follows the proposed Maas gradient flow: different dynamics can share the same equilibrium while following different transport paths and descent rates. dissipation prescribed by the Maas geometry. We measure the corresponding finite-difference rate as
\[
    D_k
    :=
    -\frac{
        \mathrm{KL}(q_{k+1}\,\|\,p)
        -
        \mathrm{KL}(q_k\,\|\,p)
    }{\Delta t}
\]
Here, \(q_k\) and \(q_{k+1}\) are the model laws before and after the update, \(p\) is the fixed target, and the negative sign makes \(D_k\) positive when KL decreases. Thus, \(D_k\) measures the observed KL decrease per unit flow time. The continuous-time Maas flow predicts the corresponding instantaneous rate to be \(\mathcal I_K(q_k\,\|\,p)\), the discrete relative Fisher information that aggregates variations of the log-density ratio across the edges of \(K\), weighted by their mobility. As shown in Figure~\ref{fig:app-tabular-diagnostics}~(bottom left), the measured and predicted rates agree to within one percent until both become dominated by numerical precision. The remaining panels verify the algebra underlying this evolution throughout the rollout. The relative residuals for mass conservation, current antisymmetry, the attraction--repulsion decomposition, and the logarithmic-mean identity remain below \(4.1\times10^{-6}\). The absolute mass-conservation defect remains at approximately \(10^{-16}\).

\begin{figure}[t]
\centering
\includegraphics[width=\linewidth]{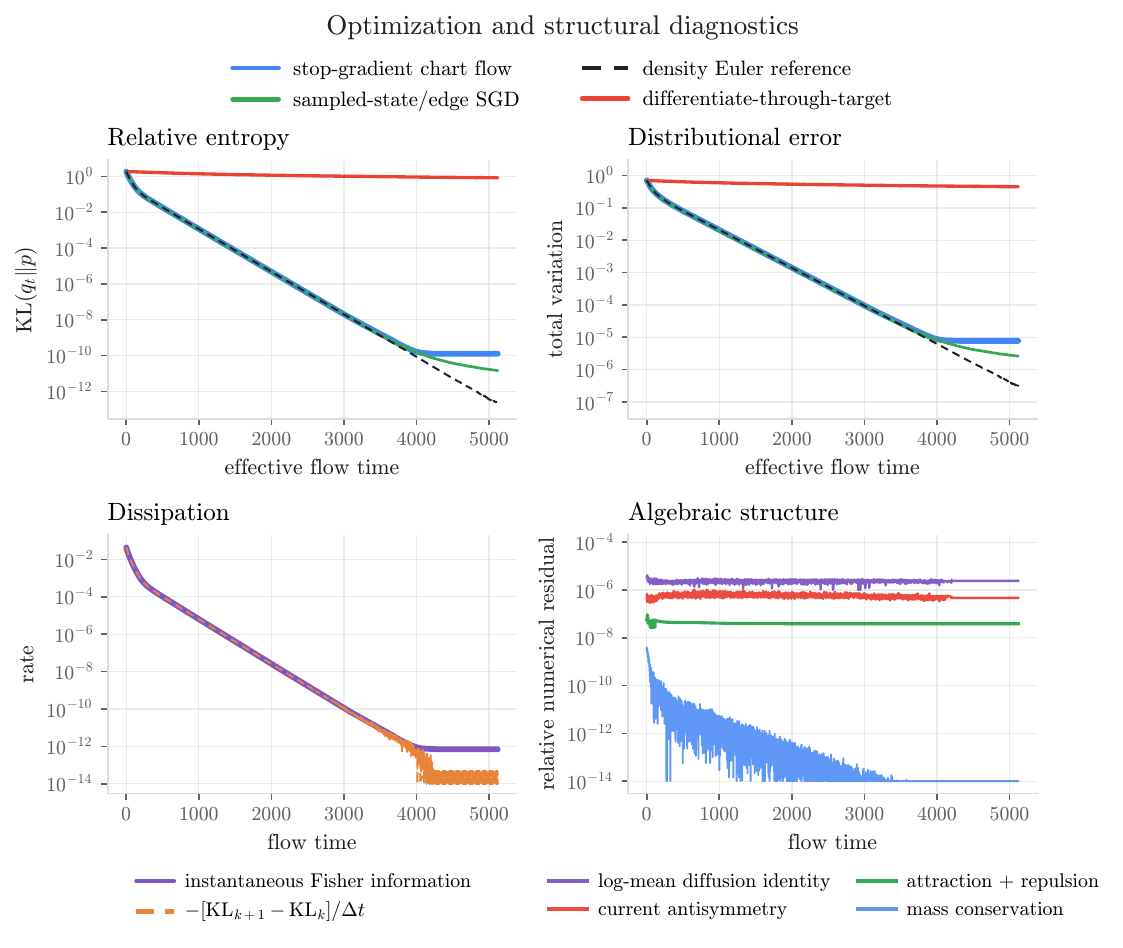}
\caption{End-to-end diagnostics for the exact tabular flow.
\emph{Top left:} relative KL for the stop-gradient chart update,
probability-space Euler reference, sampled state--edge implementation, and
differentiate-through-target control. \emph{Top right:} the corresponding
total-variation trajectories. \emph{Bottom left:} predicted instantaneous
dissipation and measured finite-difference KL decrease.
\emph{Bottom right:} numerical residuals for four structural identities.
The intended chart update, probability-space reference, and sampled
implementation agree throughout the resolved descent, whereas the
differentiate-through-target control does not.}
\label{fig:app-tabular-diagnostics}
\end{figure}

\paragraph{The chart and particle updates exhibit the predicted scaling.}
We first test the two convergence predictions governing the numerical realization. Across five decades of step size, the discrepancy between the chart update and the explicit probability-space Euler step has fitted slope
\(2.00\), verifying the \(O(\eta^2)\) local consistency predicted by \eqref{eq:chart-law-local-consistency}. Across the particle-count sweep, the empirical post-jump law converges to the exact transported law with fitted
slope \(-0.51\), matching the canonical \(O(M^{-1/2})\) Monte Carlo rate. Thus, decreasing the step size recovers the intended probability flow, while increasing the particle population recovers its exact law-level transition.
The complementary test of \eqref{eq:particle-chart-consistency}, reported in Figure~\ref{fig:particle-batch}~(left) and Table~\ref{tab:app-lift-order}, verifies the same second-order agreement between the frozen transition and the logit update. Consequently, both the chart update and the frozen particle transition reproduce the probability-space Euler evolution to first order in \(\eta\).


\paragraph{The frozen transition realizes the Euler step.}
At the representative outer step used for the particle control, the production step size remains comfortably within the admissible range of \eqref{eq:particle-cfl-condition}. At the law level, the exact post-jump distribution agrees with the explicit probability-space Euler update to machine precision, directly verifying \eqref{eq:particle-law-euler-identity}; the corresponding residuals are reported in Table~\ref{tab:app-lift-order}. At the particle level, the observed root-mean-square error is within two percent of the scale predicted by multinomial sampling, the empirical two-standard-deviation coverage is within half a percentage point of its nominal value, and no state deviates by more
than three predicted standard deviations. Together with the
\(O(M^{-1/2})\) convergence above, these diagnostics establish both parts of the particle construction: the frozen transition realizes the Euler update exactly at the law level, and finite particle populations recover that law at
the expected Monte Carlo rate.


\subsection{Finite neural amortization}
\label{app:additional-neural-results}

The preceding experiments validate the transport update without representation
or optimization error. We now isolate the approximation introduced when a
single finite-capacity generator must absorb the successive conditional
teachers. We examine this at three levels: the accuracy of each local fit, the
accumulated deviation from the conditional oracle, and the sensitivity of the
reported marginal to finite latent banks.

\paragraph{The generator learns the local transport update.}
At the final outer step, the training, held-out, and induced marginal teacher
errors are all approximately \(2.1\times10^{-4}\) in total variation. The
training and held-out conditional errors are indistinguishable even though the
held-out latents are never used for optimization, leaving no measurable
generalization gap across latent codes. The equally small marginal error shows
that this conditional accuracy transfers directly to the model law rather
than being lost when the conditionals are averaged.

\paragraph{The neural rollout preserves target-directed descent.}
Finite inner optimization does not guarantee that every fitted update
decreases KL. Nevertheless, every seed retains the strong overall descent of
the exact flow. Small one-step increases occur during the rollout, but their
combined magnitude remains below one percent of the cumulative decrease.
Finite fitting therefore introduces minor local fluctuations without
disrupting the global movement toward the target. Consistently,
Figure~\ref{fig:mlp-final} shows that every final marginal recovers the three
target modes and their unequal allocation of mass. The individual runs closely
resemble their mean, ruling out an apparently accurate average produced by
qualitatively different endpoints.

\begin{figure}[t]
\centering
\includegraphics[width=\linewidth]{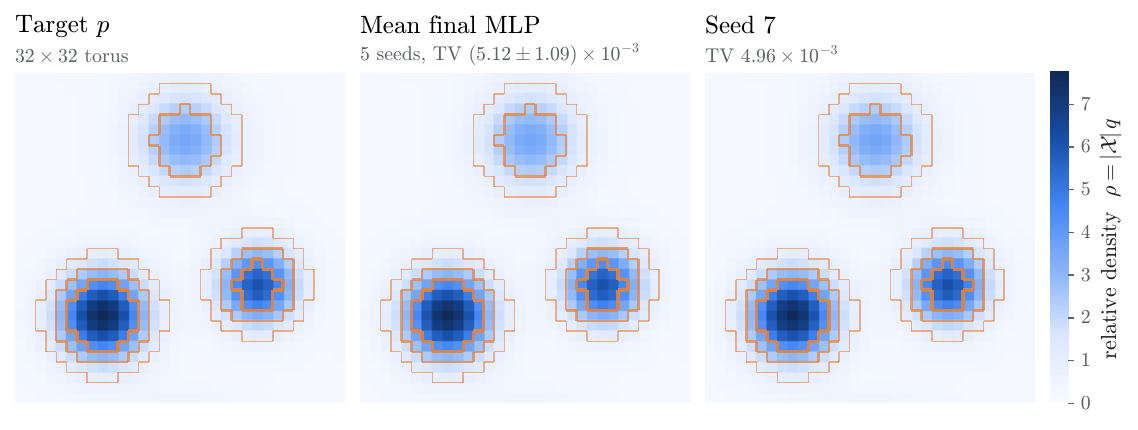}
\caption{\textbf{Final neural marginals} across five seeds. The panels show the target, the mean final marginal, and one representative seed. Orange contours mark the fixed target-density levels. Every seed recovers the three modes and
their unequal mass allocation, with limited variation across runs (final total variation $(5.12 \pm 1.09)\times10^{-3}$).}
\label{fig:mlp-final}
\end{figure}

\paragraph{Larger fitting batches improve oracle tracking.}
The conditional oracle applies every transport teacher exactly, so its
difference from the neural rollout isolates error accumulated through
amortization rather than error in the transport itself.
Figure~\ref{fig:particle-batch}~(right) shows that the final neural--oracle gap
falls sharply as the fitting batch grows, across all tested step sizes, before
levelling off at a small residual. The time-averaged gap follows the same
pattern, showing that the improvement holds throughout the rollout rather than
only at its endpoint. Because the full-population condition removes latent
subsampling but retains a small gap, the remaining discrepancy can be
attributed to finite representation and optimization rather than to the
teacher construction.

\paragraph{Endpoint estimates are stable across latent banks.}
We finally reevaluate each trained generator using progressively larger fresh
latent banks, without changing its parameters. The resulting variation in the
estimated marginal decreases with bank size and remains well below the final
distance to the target even for the smallest evaluation bank. The reported
endpoint error is therefore not an artefact of Monte Carlo variation in the
evaluation marginal. This diagnostic is distinct from the preceding batch
sweep: fitting-batch size changes the transport updates learned during the
rollout, whereas evaluation-bank size changes only how the marginal of an
already trained generator is estimated.

%% file: sections/appendix/appendix_proof.tex
\section{Proofs}
\label{app:discrete-drifting-proofs}

Throughout this appendix, the target probability mass function \(p\) is
fixed and strictly positive. For any strictly positive probability mass
function \(q\) on \(\X\), we write
\begin{equation}
\label{eq:appendix-relative-log-density}
    \ell_{p,q}(x)
    :=
    \log\frac{\rho_q(x)}{\rho_p(x)},
    \qquad x\in\X,
\end{equation}
and, for any \(h:\X\to\mathbb R\),
\begin{equation}
\label{eq:q-expectation-notation}
    \langle h\rangle_q
    :=
    \sum_{x\in\X} q(x)h(x).
\end{equation}

The proofs rely repeatedly on three elementary identities: the differential
of the logit chart, the representation of the Maas current in logit
coordinates, and the fact that the resulting logit drift is centered under
the current model law. We collect them first.

\subsection{Preliminary identities}
Let \(\Phi:\mathbb R^\X\to\mathcal P_*(\X)\) denote the logit chart \(\Phi(g)=q_g\) defined in~\eqref{eq:pi-logit-parameterization}. For \(g,h\in\mathbb R^\X\), we write \(D\Phi(g)[h]\) for the directional derivative of \(\Phi\) at \(g\) in the direction \(h\), namely \[ D\Phi(g)[h] := \left. \frac{\mathrm d}{\mathrm d\varepsilon} \Phi(g+\varepsilon h) \right|_{\varepsilon=0}. \] Equivalently, for each \(x\in\X\), \[ D\Phi(g)[h](x) = \left. \frac{\mathrm d}{\mathrm d\varepsilon} q_{g+\varepsilon h}(x) \right|_{\varepsilon=0}. \]
\begin{lemma}[Logit differential and current identities]
\label{lem:logit-current-identities}
For every
\(g,h\in\mathbb R^\X\),
\begin{equation}
\label{eq:appendix-logit-differential}
    D\Phi(g)[h](x)
    =
    q_g(x)
    \left(
        h(x)-\langle h\rangle_{q_g}
    \right),
    \qquad x\in\X.
\end{equation}
Moreover, the current admits the logit-coordinate representation
\begin{equation}
\label{eq:appendix-current-logit-form}
    J_{p,q_g}(x,y)
    =
    q_g(x)K(x,y)L\bigl(\delta_g(x,y)\bigr)
    \left(
        s_p(x,y)-\delta_g(x,y)
    \right),
\end{equation}
and therefore
\begin{equation}
\label{eq:appendix-law-logit-identity}
    -\sum_{y\in\X}J_{p,q_g}(x,y)
    =
    q_g(x)\Delta_g(x).
\end{equation}
In particular, the logit drift is centered under the current model law:
\begin{equation}
\label{eq:appendix-drift-centering}
    \langle\Delta_g\rangle_{q_g}
    =
    \sum_{x\in\X}q_g(x)\Delta_g(x)
    =
    0.
\end{equation}
\end{lemma}

\begin{proof}
We first differentiate the logit parametrization. For a perturbation
\(g+\varepsilon h\),
\[
    q_{g+\varepsilon h}(x)
    =
    \frac{\pi(x)e^{g(x)+\varepsilon h(x)}}
    {Z_\pi(g+\varepsilon h)}.
\]
Differentiating at \(\varepsilon=0\) gives
\begin{align*}
    D\Phi(g)[h](x)
    &=
    \frac{\pi(x)e^{g(x)}h(x)}{Z_\pi(g)}
    -
    \frac{\pi(x)e^{g(x)}}{Z_\pi(g)^2}
    \sum_{z\in\X}\pi(z)e^{g(z)}h(z)
    \\
    &=
    q_g(x)
    \left(
        h(x)-\sum_{z\in\X}q_g(z)h(z)
    \right),
\end{align*}
which is~\eqref{eq:appendix-logit-differential}. The second term is the
correction induced by normalization; in particular,
\(D\Phi(g)[\mathbf 1]=0\), reflecting the gauge invariance
\(q_{g+c\mathbf 1}=q_g\).

We next express the current in the same coordinates. By positive
\(1\)-homogeneity of the logarithmic mean,
\[
    m_{\rho_g}(x,y)
    =
    \rho_g(x)
    m\left(
        1,\frac{\rho_g(y)}{\rho_g(x)}
    \right).
\]
Since
\[
    \frac{\rho_g(y)}{\rho_g(x)}
    =
    e^{g(y)-g(x)}
    =
    e^{\delta_g(x,y)},
\]
the definition of \(L\) yields
\begin{equation}
\label{eq:appendix-mobility-logit-form}
    m_{\rho_g}(x,y)
    =
    \rho_g(x)L\bigl(\delta_g(x,y)\bigr).
\end{equation}
At the same time,
\begin{align*}
    \ell_{p,q_g}(x)-\ell_{p,q_g}(y)
    &=
    \bigl(
        \log\rho_g(x)-\log\rho_g(y)
    \bigr)
    -
    \bigl(
        \log\rho_p(x)-\log\rho_p(y)
    \bigr)
    \\
    &=
    s_p(x,y)-\delta_g(x,y).
\end{align*}
Substituting these identities into~\eqref{eq:conductance-and-current}, and
using \(q_g(x)=\pi(x)\rho_g(x)\), gives
\eqref{eq:appendix-current-logit-form}. Hence
\begin{align*}
    -\sum_{y\in\X}J_{p,q_g}(x,y)
    &=
    q_g(x)
    \sum_{y\in\X}
        K(x,y)L\bigl(\delta_g(x,y)\bigr)
        \left(
            \delta_g(x,y)-s_p(x,y)
        \right)
    \\
    &=
    q_g(x)\Delta_g(x),
\end{align*}
which proves~\eqref{eq:appendix-law-logit-identity}.

It remains to establish the centering identity. Reversibility of \(K\) and
symmetry of the logarithmic mean give
\[
    \pi(x)K(x,y)m_{\rho_g}(x,y)
    =
    \pi(y)K(y,x)m_{\rho_g}(y,x).
\]
Since
\[
    \ell_{p,q_g}(y)-\ell_{p,q_g}(x)
    =
    -
    \bigl(
        \ell_{p,q_g}(x)-\ell_{p,q_g}(y)
    \bigr),
\]
the current is antisymmetric:
\begin{equation}
\label{eq:appendix-current-antisymmetry}
    J_{p,q_g}(y,x)
    =
    -J_{p,q_g}(x,y).
\end{equation}
Summing~\eqref{eq:appendix-law-logit-identity} over \(x\) therefore gives
\begin{align*}
    \sum_{x\in\X}q_g(x)\Delta_g(x)
    &=
    -\sum_{x,y\in\X}J_{p,q_g}(x,y)
    \\
    &=
    -\frac12
    \sum_{x,y\in\X}
    \left(
        J_{p,q_g}(x,y)+J_{p,q_g}(y,x)
    \right)
    \\
    &=0.
\end{align*}
This is~\eqref{eq:appendix-drift-centering}.
\end{proof}

\subsection{Logit representation of the gradient flow}

\begin{proof}[Proof of Proposition~\ref{prop:normalization-free-flow}]
Let \(q_t=q_{g_t}\). Applying the chain rule to the logit chart and using
Lemma~\ref{lem:logit-current-identities} gives
\begin{equation}
\label{eq:appendix-time-chart-differential}
    \dot q_t(x)
    =
    q_t(x)
    \left(
        \dot g_t(x)
        -
        \langle\dot g_t\rangle_{q_t}
    \right).
\end{equation}

Suppose first that \(q_t\) satisfies the probability-coordinate gradient
flow in~\eqref{eq:target-kl-gradient-flow}. By
\eqref{eq:appendix-law-logit-identity},
\begin{equation}
\label{eq:appendix-probability-flow-logit-form}
    \dot q_t(x)
    =
    q_t(x)\Delta_{g_t}(x).
\end{equation}
Since \(q_t(x)>0\) for every \(x\in\X\), comparison with
\eqref{eq:appendix-time-chart-differential} yields
\[
    \dot g_t(x)
    -
    \langle\dot g_t\rangle_{q_t}
    =
    \Delta_{g_t}(x).
\]
Thus
\[
    \dot g_t
    =
    \Delta_{g_t}
    +
    c_t\mathbf 1,
    \qquad
    c_t
    :=
    \langle\dot g_t\rangle_{q_t}.
\]
The scalar \(c_t\) is continuous because \(t\mapsto g_t\) is continuously
differentiable and \(g\mapsto q_g\) is smooth. 

Conversely, suppose that
\[
    \dot g_t
    =
    \Delta_{g_t}+c_t\mathbf 1
\]
for some continuous scalar function \(c_t\). The centering identity
\eqref{eq:appendix-drift-centering} implies
\[
    \langle\dot g_t\rangle_{q_t}
    =
    \langle\Delta_{g_t}\rangle_{q_t}+c_t
    =
    c_t.
\]
Substituting this into~\eqref{eq:appendix-time-chart-differential} eliminates
the gauge term and gives
\begin{align*}
    \dot q_t(x)
    &=
    q_t(x)\Delta_{g_t}(x)
    \\
    &=
    -\sum_{y\in\X}J_{p,q_t}(x,y),
\end{align*}
where the final identity again follows from
\eqref{eq:appendix-law-logit-identity}. Hence \(q_t\) satisfies
\eqref{eq:target-kl-gradient-flow}, completing the equivalence.
\end{proof}

\subsection{Local consistency and entropy descent}

\begin{proof}[Proof of Proposition~\ref{prop:chart-euler-consistency}]
Fix \(g\in\mathbb R^\X\), and abbreviate
\[
    q:=q_g,
    \qquad
    \Delta:=\Delta_g,
    \qquad
    \ell(x)
    :=
    \log\frac{\rho_q(x)}{\rho_p(x)}.
\]

We first establish the first-order evolution of the probability law. Recall
that \(\Phi(g)=q_g\). By definition of the directional derivative,
\[
    D\Phi(g)[\Delta]
    =
    \left.
    \frac{\mathrm d}{\mathrm d\varepsilon}
    q_{g+\varepsilon\Delta}
    \right|_{\varepsilon=0}.
\]
Since \(\Phi\) is smooth, Taylor's theorem gives
\begin{equation}
\label{eq:appendix-chart-taylor}
    q_{g+\eta\Delta}
    =
    q_g
    +
    \eta D\Phi(g)[\Delta]
    +
    O(\eta^2).
\end{equation}
By~\eqref{eq:appendix-logit-differential},
\[
    D\Phi(g)[\Delta](x)
    =
    q(x)
    \left(
        \Delta(x)-\langle\Delta\rangle_q
    \right).
\]
The centering identity~\eqref{eq:appendix-drift-centering} gives
\(\langle\Delta\rangle_q=0\), and hence
\[
    D\Phi(g)[\Delta](x)
    =
    q(x)\Delta(x).
\]
Substituting this into~\eqref{eq:appendix-chart-taylor} yields
\[
    q_{g+\eta\Delta_g}
    =
    q_g
    +
    \eta q_g\odot\Delta_g
    +
    O(\eta^2),
\]
which proves~\eqref{eq:chart-law-local-consistency}.

We next study the corresponding change in relative entropy. Define
\[
    \mathcal H_p(q)
    :=
    \mathrm{KL}(q\,\|\,p)
    =
    \sum_{x\in\X}
        q(x)\log\frac{q(x)}{p(x)},
\]
and
\[
    \mathcal G(g)
    :=
    \mathcal H_p(q_g).
\]
Thus \(\mathcal G=\mathcal H_p\circ\Phi\).

For a perturbation \(v\in\mathbb R^\X\), the directional derivative of
\(\mathcal H_p\) at \(q\) is
\begin{align*}
    D\mathcal H_p(q)[v]
    &:=
    \left.
    \frac{\mathrm d}{\mathrm d\varepsilon}
    \mathcal H_p(q+\varepsilon v)
    \right|_{\varepsilon=0}
    \\
    &=
    \left.
    \frac{\mathrm d}{\mathrm d\varepsilon}
    \sum_{x\in\X}
        \bigl(q(x)+\varepsilon v(x)\bigr)
        \log
        \frac{q(x)+\varepsilon v(x)}{p(x)}
    \right|_{\varepsilon=0}
    \\
    &=
    \sum_{x\in\X}
        v(x)
        \left(
            \log\frac{q(x)}{p(x)}+1
        \right).
\end{align*}

Now let \(h\in\mathbb R^\X\) be an arbitrary perturbation of the logits.
Since
\[
    q_{g+\varepsilon h}
    =
    q
    +
    \varepsilon D\Phi(g)[h]
    +
    o(\varepsilon),
\]
we obtain
\begin{align}
    D\mathcal G(g)[h]
    &:=
    \left.
    \frac{\mathrm d}{\mathrm d\varepsilon}
    \mathcal H_p(q_{g+\varepsilon h})
    \right|_{\varepsilon=0}
    \nonumber\\
    &=
    D\mathcal H_p(q)
    \bigl[
        D\Phi(g)[h]
    \bigr]
    \nonumber\\
    &=
    \sum_{x\in\X}
        D\Phi(g)[h](x)
        \left(
            \log\frac{q(x)}{p(x)}+1
        \right).
\label{eq:appendix-KL-first-variation}
\end{align}

Because \(q_g\) is normalized for every \(g\),
\[
    \sum_{x\in\X}q_g(x)=1.
\]
Taking the directional derivative of this identity in the direction \(h\)
gives
\[
    \sum_{x\in\X}D\Phi(g)[h](x)=0.
\]
Hence the constant \(+1\) in
\eqref{eq:appendix-KL-first-variation} makes no contribution, and
\[
    D\mathcal G(g)[h]
    =
    \sum_{x\in\X}
        D\Phi(g)[h](x)
        \log\frac{q(x)}{p(x)}.
\]
Since
\[
    \frac{q(x)}{p(x)}
    =
    \frac{\pi(x)\rho_q(x)}
         {\pi(x)\rho_p(x)}
    =
    \frac{\rho_q(x)}{\rho_p(x)},
\]
we have
\[
    \log\frac{q(x)}{p(x)}
    =
    \ell(x).
\]

Taking \(h=\Delta\), and using
\(D\Phi(g)[\Delta](x)=q(x)\Delta(x)\), therefore gives
\begin{align*}
    D\mathcal G(g)[\Delta]
    &=
    \sum_{x\in\X}
        q(x)\Delta(x)\ell(x)
    \\
    &=
    -\sum_{x,y\in\X}
        J_{p,q}(x,y)\ell(x),
\end{align*}
where the second equality follows from
\eqref{eq:appendix-law-logit-identity}.

We now use antisymmetry of the current. Relabelling \(x\) and \(y\) in the
double sum gives
\begin{align*}
    \sum_{x,y\in\X}
        J_{p,q}(x,y)\ell(y)
    &=
    \sum_{x,y\in\X}
        J_{p,q}(y,x)\ell(x)
    \\
    &=
    -\sum_{x,y\in\X}
        J_{p,q}(x,y)\ell(x).
\end{align*}
Therefore,
\begin{align*}
    \sum_{x,y\in\X}
        J_{p,q}(x,y)\ell(x)
    &=
    \frac12
    \sum_{x,y\in\X}
        J_{p,q}(x,y)
        \bigl(
            \ell(x)-\ell(y)
        \bigr).
\end{align*}
Using the definition of the current
\eqref{eq:conductance-and-current}, we obtain
\begin{align*}
    D\mathcal G(g)[\Delta]
    &=
    -\frac12
    \sum_{x,y\in\X}
        \pi(x)K(x,y)m_{\rho_q}(x,y)
        \bigl(
            \ell(x)-\ell(y)
        \bigr)^2
    \\
    &=
    -\mathcal I_K(q\,\|\,p).
\end{align*}

Since \(\mathcal G\) is smooth, Taylor's theorem now gives
\begin{align*}
    \mathrm{KL}
    \bigl(
        q_{g+\eta\Delta_g}\,\|\,p
    \bigr)
    &=
    \mathcal G(g+\eta\Delta)
    \\
    &=
    \mathcal G(g)
    +
    \eta D\mathcal G(g)[\Delta]
    +
    O(\eta^2)
    \\
    &=
    \mathrm{KL}(q_g\,\|\,p)
    -
    \eta\mathcal I_K(q_g\,\|\,p)
    +
    O(\eta^2),
\end{align*}
which proves~\eqref{eq:chart-kl-local-descent}.

It remains to establish strict descent away from the target. Every summand
in \(\mathcal I_K(q\,\|\,p)\) is nonnegative. Moreover, since
\(\pi(x)>0\) for every \(x\in\X\) and \(m_{\rho_q}(x,y)>0\) whenever
\(\rho_q(x),\rho_q(y)>0\), the identity
\[
    \mathcal I_K(q\,\|\,p)=0
\]
implies that, for every pair \(x,y\in\X\) such that \(K(x,y)>0\),
\[
    \bigl(\ell(x)-\ell(y)\bigr)^2=0.
\]
Hence
\[
    \ell(x)=\ell(y)
    \qquad
    \text{whenever }K(x,y)>0.
\]
By irreducibility, for any \(x,y\in\X\) there exists a finite path
\[
    x=x_0,x_1,\ldots,x_n=y
\]
such that
\[
    K(x_i,x_{i+1})>0,
    \qquad
    i=0,\ldots,n-1.
\]
Hence
\[
    \ell(x_0)
    =
    \ell(x_1)
    =
    \cdots
    =
    \ell(x_n),
\]
and therefore \(\ell(x)=\ell(y)\). Since \(x\) and \(y\) are arbitrary,
\(\ell\) is constant on \(\X\).

Thus, for some \(a\in\mathbb R\),
\[
    \log\frac{\rho_q(x)}{\rho_p(x)}
    =
    a
    \qquad
    \text{for all }x\in\X,
\]
or equivalently
\[
    q(x)=e^a p(x)
    \qquad
    \text{for all }x\in\X.
\]
Since both \(q\) and \(p\) are probability distributions,
\[
    1
    =
    \sum_x q(x)
    =
    e^a\sum_x p(x)
    =
    e^a,
\]
so \(a=0\) and \(q=p\). Hence
\[
    \mathcal I_K(q\,\|\,p)=0
    \quad\Longleftrightarrow\quad
    q=p.
\]

Consequently, if \(q_g\neq p\), then
\(\mathcal I_K(q_g\,\|\,p)>0\). For sufficiently small \(|\eta|\), write
the Taylor remainder as \(R(\eta)\), with
\[
    |R(\eta)|
    \leq
    C\eta^2
\]
for some constant \(C>0\). Then
\begin{align*}
    \mathrm{KL}
    \bigl(
        q_{g+\eta\Delta_g}\,\|\,p
    \bigr)
    -
    \mathrm{KL}(q_g\,\|\,p)
    &=
    -\eta\mathcal I_K(q_g\,\|\,p)
    +
    R(\eta)
    \\
    &\leq
    -\eta\mathcal I_K(q_g\,\|\,p)
    +
    C\eta^2
    \\
    &=
    -\eta
    \left(
        \mathcal I_K(q_g\,\|\,p)-C\eta
    \right).
\end{align*}
For
\[
    0<\eta<
    \frac{\mathcal I_K(q_g\,\|\,p)}{C},
\]
the right-hand side is strictly negative. Thus the update strictly decreases
\(\mathrm{KL}(q_g\,\|\,p)\) for every sufficiently small positive
\(\eta\).
\end{proof}

\subsection{Particle realization and minimality}

The particle construction rests on the elementary observation that an
antisymmetric signed current can be represented by nonnegative directional
jump rates by assigning each edge only its positive current component. The
following lemma records the resulting generator and its induced
law-level evolution.

\begin{lemma}[Properties of the positive-part generator]
\label{lem:positive-part-generator}
Let \(p\) and \(q\) be strictly positive probability mass functions, and
let \(A_{p,q}\) be defined by
\eqref{eq:minimal-jump-rate}--\eqref{eq:particle-generator}. Then
\(A_{p,q}\) is a continuous-time Markov generator and
\begin{equation}
\label{eq:appendix-generator-law-identity}
    (qA_{p,q})(x)
    =
    -\sum_{y\in\X}J_{p,q}(x,y)
    =
    q(x)\Delta_{p,q}(x).
\end{equation}
If~\eqref{eq:particle-cfl-condition} holds, then
\(P_{\eta,p,q}=I+\eta A_{p,q}\) is a Markov transition kernel and
\begin{equation}
\label{eq:appendix-particle-euler-identity}
    qP_{\eta,p,q}
    =
    q+\eta q\odot\Delta_{p,q}.
\end{equation}
\end{lemma}

\begin{proof}
By construction, the off-diagonal entries of \(A_{p,q}\) are nonnegative,
while its diagonal entries are chosen so that every row sums to zero.
Hence \(A_{p,q}\) is a continuous-time Markov generator.

We regard probability distributions as row vectors. For every \(x\in\X\),
\begin{align*}
    (qA_{p,q})(x) &= \sum_{y\in \mathcal{X}} q(y) A_{p,q}(y,x)\\ 
    &=
    \sum_{y\neq x}
        q(y)a_{p,q}(y,x)
    -
    q(x)
    \sum_{y\neq x}
        a_{p,q}(x,y)
    \\
    &=
    \sum_{y\neq x}
    \left(
        [J_{p,q}(y,x)]_+
        -
        [J_{p,q}(x,y)]_+
    \right).
\end{align*}
Using antisymmetry,
\(J_{p,q}(y,x)=-J_{p,q}(x,y)\), together with the scalar identity
\[
    [-r]_+-[r]_+
    =
    -r,
\]
we obtain
\begin{align*}
    (qA_{p,q})(x)
    &=
    -\sum_{y\neq x}J_{p,q}(x,y)
    \\
    &=
    -\sum_{y\in\X}J_{p,q}(x,y)
    \\
    &=
    q(x)\Delta_{p,q}(x),
\end{align*}
where \(J_{p,q}(x,x)=0\). This proves
\eqref{eq:appendix-generator-law-identity}.

It remains to verify the discrete-time step. For \(x\neq y\),
\[
    P_{\eta,p,q}(x,y)
    =
    \eta a_{p,q}(x,y)
    \geq0,
\]
while
\[
    P_{\eta,p,q}(x,x)
    =
    1-
    \eta\sum_{y\neq x}a_{p,q}(x,y)
    \geq0
\]
by~\eqref{eq:particle-cfl-condition}. The rows of \(P_{\eta,p,q}\) sum to
one because the rows of \(A_{p,q}\) sum to zero. Thus \(P_{\eta,p,q}\) is
a Markov transition kernel. Finally,
\[
    qP_{\eta,p,q}
    =
    q(I+\eta A_{p,q})
    =
    q+\eta qA_{p,q}
    =
    q+\eta q\odot\Delta_{p,q},
\]
which proves~\eqref{eq:appendix-particle-euler-identity}.
\end{proof}

\begin{proof}[Proof of Proposition~\ref{prop:particle-lift}]
For \(q=q_g\), define
\begin{equation}
\label{eq:appendix-p-eta-g-definition}
    A_{p,g}
    :=
    A_{p,q_g},
    \qquad
    P_{\eta,g}
    :=
    P_{\eta,p,q_g}
    =
    I+\eta A_{p,q_g}.
\end{equation}
By Lemma~\ref{lem:positive-part-generator} and the step-size condition~\eqref{eq:particle-cfl-condition}, evaluated at $q = q_{g}$, 
\(P_{\eta,g}\) is a Markov transition kernel.

The law-level and logit-level drifts coincide at \(q_g\). Indeed,
\[
    \rho_g(x)
    =
    \frac{e^{g(x)}}{Z_\pi(g)}
\]
implies
\begin{align*}
    \delta_{q_g}(x,y)
    &=
    \log\rho_g(y)-\log\rho_g(x)
    \\
    &=
    g(y)-g(x)
    \\
    &=
    \delta_g(x,y).
\end{align*}
Using~\eqref{eq:law-level-current}, we therefore obtain
\begin{align*}
    \Delta_{p,q_g}(x)
    &=
    -\frac{1}{q_g(x)}
    \sum_{y\in\X}J_{p,q_g}(x,y)
    \\
    &=
    \sum_{y\in\X}
        K(x,y)L\bigl(\delta_g(x,y)\bigr)
        \left(
            \delta_g(x,y)-s_p(x,y)
        \right)
    \\
    &=
    \Delta_g(x).
\end{align*}

The exact law-level Euler identity from
Lemma~\ref{lem:positive-part-generator} consequently becomes
\begin{equation}
\label{eq:appendix-particle-law-at-g}
    q_gP_{\eta,g}
    =
    q_g+\eta q_g\odot\Delta_g.
\end{equation}
On the other hand, Proposition~\ref{prop:chart-euler-consistency} gives
\begin{equation}
\label{eq:appendix-logit-law-at-g}
    q_{g+\eta\Delta_g}
    =
    q_g+\eta q_g\odot\Delta_g
    +
    O(\eta^2).
\end{equation}
Comparing the two expressions yields
\[
    q_{g+\eta\Delta_g}
    =
    q_gP_{\eta,g}
    +
    O(\eta^2),
\]
which proves~\eqref{eq:particle-chart-consistency}.
\end{proof}

The positive-part realization has an additional useful property: among all
nonnegative pairs of directional rates producing the same signed current
across an edge, it introduces the least total particle traffic. We make this
precise next.

\begin{proposition}[Edgewise minimality of the positive-part realization]
\label{prop:edgewise-minimality}
Fix strictly positive probability mass functions \(p,q\) and distinct
states \(x,y\in\X\). Suppose that \(b(x,y),b(y,x)\geq0\) are alternative
jump rates satisfying
\begin{equation}
\label{eq:appendix-alternative-current-constraint}
    q(x)b(x,y)-q(y)b(y,x)
    =
    J_{p,q}(x,y).
\end{equation}
Then there exists a unique \(r_{x,y}\geq0\) such that
\begin{align}
    q(x)b(x,y)
    &=
    [J_{p,q}(x,y)]_+
    +
    r_{x,y},
\label{eq:appendix-forward-counterflow}
    \\
    q(y)b(y,x)
    &=
    [-J_{p,q}(x,y)]_+
    +
    r_{x,y}.
\label{eq:appendix-backward-counterflow}
\end{align}
Consequently, the total probability traffic across the edge \(\{x,y\}\) satisfies
\begin{equation}
\label{eq:appendix-edge-activity-lower-bound}
    q(x)b(x,y)+q(y)b(y,x)
    =
    |J_{p,q}(x,y)|+2r_{x,y}
    \geq
    |J_{p,q}(x,y)|.
\end{equation}
Equality holds if and only if
\[
    b(x,y)=a_{p,q}(x,y),
    \qquad
    b(y,x)=a_{p,q}(y,x).
\]
\end{proposition}

\begin{proof}
Set
\[
    u:=q(x)b(x,y),
    \qquad
    v:=q(y)b(y,x),
    \qquad
    J:=J_{p,q}(x,y).
\]
Then \(u,v\geq0\), and the current constraint becomes
\[
    u-v=J.
\]

We first characterize all nonnegative pairs \((u,v)\) satisfying this
constraint. If \(J\geq0\), then
\[
    [J]_+=J,
    \qquad
    [-J]_+=0,
\]
and \(u-v=J\) implies \(u=J+v\). Hence, setting
\(r_{x,y}:=v\geq0\),
\[
    u=[J]_++r_{x,y},
    \qquad
    v=[-J]_++r_{x,y}.
\]
If \(J<0\), then
\[
    [J]_+=0,
    \qquad
    [-J]_+=-J,
\]
and \(u-v=J\) is equivalently \(v=-J+u\). Setting
\(r_{x,y}:=u\geq0\) again gives
\[
    u=[J]_++r_{x,y},
    \qquad
    v=[-J]_++r_{x,y}.
\]
Thus, in either case,
\[
    r_{x,y}=\min\{u,v\},
\]
which also shows that \(r_{x,y}\) is uniquely determined.

Adding the two directional traffic terms yields
\[
    u+v
    =
    [J]_++[-J]_++2r_{x,y}
    =
    |J|+2r_{x,y}.
\]
Since \(r_{x,y}\geq0\), the total traffic is bounded below by \(|J|\),
with equality if and only if \(r_{x,y}=0\).

When \(r_{x,y}=0\),
\[
    u=[J]_+,
    \qquad
    v=[-J]_+.
\]
Recalling the definitions of \(u\) and \(v\), this gives
\[
    b(x,y)
    =
    \frac{[J_{p,q}(x,y)]_+}{q(x)}
    =
    a_{p,q}(x,y),
\]
and, using antisymmetry of the current,
\[
    b(y,x)
    =
    \frac{[-J_{p,q}(x,y)]_+}{q(y)}
    =
    \frac{[J_{p,q}(y,x)]_+}{q(y)}
    =
    a_{p,q}(y,x).
\]
Therefore the positive-part construction is the unique edgewise
realization attaining the minimum total directional traffic.
\end{proof}

The minimal realization also has the expected behavior at equilibrium:
once the model law reaches the target, not only does the marginal
distribution stop evolving, but every particle jump rate vanishes.

\begin{corollary}[Pathwise stationarity at the target]
\label{cor:pathwise-stationarity-at-target}
If \(q=p\), then
\[
    J_{p,p}(x,y)=0,
    \qquad
    a_{p,p}(x,y)=0
\]
for all \(x\neq y\). Hence
\[
    A_{p,p}=0,
    \qquad
    P_{\eta,p,p}=I,
\]
and particles remain at their current states almost surely.
\end{corollary}

\begin{proof}
If \(q=p\), then \(\rho_q=\rho_p\), and hence
\[
    \log\frac{\rho_q(x)}{\rho_p(x)}
    =
    0
    \qquad
    \text{for every }x\in\X.
\]
The current~\eqref{eq:conductance-and-current} therefore vanishes
identically. It follows immediately from
\eqref{eq:minimal-jump-rate} that every off-diagonal jump rate is zero, so
\(A_{p,p}=0\) and \(P_{\eta,p,p}=I\). Thus the particle realization is
pathwise stationary at the target.
\end{proof}